\documentclass{article}

\usepackage{microtype} 
\usepackage{graphicx}
\usepackage{subfigure}
\usepackage{booktabs} 

\usepackage{hyperref}

\usepackage{amsthm}
\usepackage{amsmath}
\usepackage{amsfonts, bm, multicol}
\usepackage{bbm, amsthm}

\newtheorem{theorem}{Theorem}
\newtheorem{prop}{Proposition}

\newtheorem{lemma}[theorem]{Lemma}

\newcommand{\diag}{\mathop{\mathrm{diag}}}
\newcommand*{\bfrac}[2]{\genfrac{\lbrace}{\rbrace}{0pt}{}{#1}{#2}}

\usepackage[accepted]{icml2021}

\icmltitlerunning{Recomposing the Reinforcement Learning Building-Blocks
 with Hypernetworks}

\begin{document}

%

%

\twocolumn[
\icmltitle{Recomposing the Reinforcement Learning Building Blocks with Hypernetworks}


\icmlsetsymbol{equal}{*}

\begin{icmlauthorlist}
\icmlauthor{Shai Keynan*}{to}
\icmlauthor{Elad Sarafian*}{to}
\icmlauthor{Sarit Kraus}{to}
\end{icmlauthorlist}

\icmlaffiliation{to}{Department of Computer Science, Bar-Ilan University, Ramat-Gan, Israel}

\icmlcorrespondingauthor{Shai Keynan, Elad Sarafian}{shai.keynan@gmail.com, elad.sarafian@gmail.com}

\icmlkeywords{Reinforcement Learning, Meta Learning, Neural Architecture, Hypernetworks, ICML}

\vskip 0.3in
]

\printAffiliationsAndNotice{\icmlEqualContribution}

\begin{abstract}
The Reinforcement Learning (RL) building blocks, i.e. $Q$-functions and policy networks, usually take elements from the cartesian product of two domains as input. In particular, the input of the $Q$-function is both the state and the action, and in multi-task problems (Meta-RL) the policy can take a state and a context. Standard architectures tend to ignore these variables' underlying interpretations and simply concatenate their features into a single vector. In this work, we argue that this choice may lead to poor gradient estimation in actor-critic algorithms and high variance learning steps in Meta-RL algorithms. To consider the interaction between the input variables, we suggest using a Hypernetwork architecture where a primary network determines the weights of a conditional dynamic network. We show that this approach improves the gradient approximation and reduces the learning step variance, which both accelerates learning and improves the final performance. We demonstrate a consistent improvement across different locomotion tasks and different algorithms both in RL (TD3 and SAC) and in Meta-RL (MAML and PEARL).
\end{abstract}

\begin{figure}[!ht]
\begin{center}
    \includegraphics[width=\linewidth]{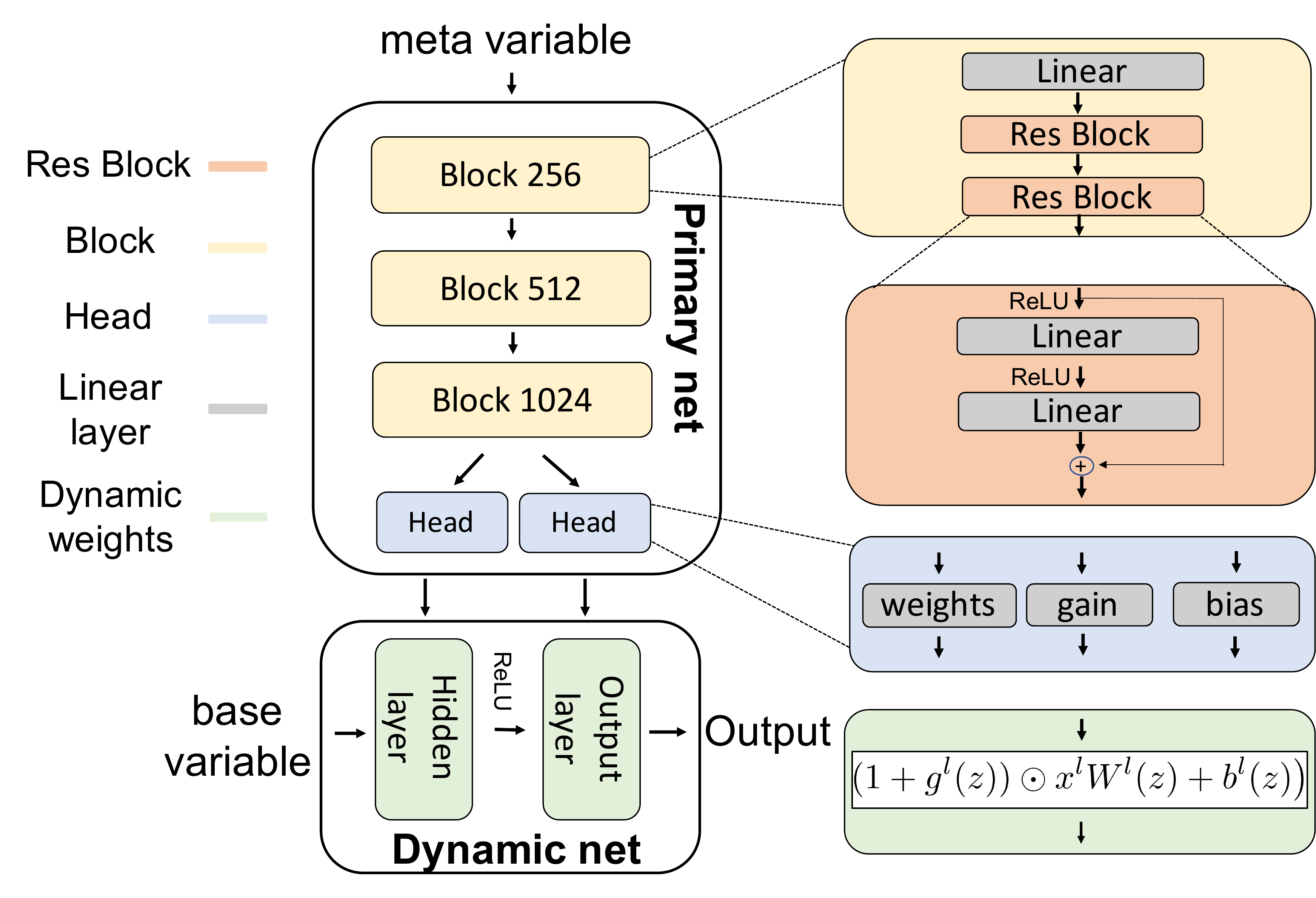}
\end{center}
\caption{The Hypernetwork architecture}
\label{fig:HyperSceme}
\end{figure}

\section{Introduction}

The rapid development of deep neural-networks as general-purpose function approximators has propelled the recent Reinforcement Learning (RL) renaissance \cite{zai2020deep}. RL algorithms have progressed in robustness, e.g. from \cite{lillicrap2016continuous} to \cite{fujimoto2018addressing}; exploration \cite{haarnoja2018soft}; gradient sampling \cite{schulman2017proximal, schulman2015trust}; and off-policy learning \cite{fujimoto2019off, kumar2019stabilizing}. Many actor-critic algorithms have focused on improving the critic learning routines by modifying the target value \cite{hasselt2016deep}, which enables more accurate and robust $Q$-function approximations. While this greatly improves the policy optimization efficiency, the performance is still bound by the networks' ability to represent $Q$-functions and policies. Such a constraint calls for studying and designing neural models suited for the representation of these RL building blocks.

A critical insight in designing neural models for RL is the reciprocity between the state and the action, which both serve as the input for the $Q$-function. At the start, each input can be processed individually according to its source domain. For example, when $s$ is a vector of images, it is common to employ CNN models \cite{kaiser2019model}, and when $s$ or $a$ are natural language words, each input can be processed separately with embedding vectors \cite{he2016deep}. The common practice in incorporating the state and action learnable features into a single network is to concatenate the two vectors and follow with MLP to yield the $Q$-value \cite{schulman2017proximal}. In this work, we argue that for actor-critic RL algorithms \cite{grondman2012survey}, such an off-the-shelf method could be significantly improved with Hypernetworks.

In actor-critic methods, for each state, sampled from the dataset distribution, the actor's task is to solve an optimization problem over the action distribution, i.e. the policy. This motivates an architecture where the $Q$-function is explicitly modeled as the value function of a contextual bandit \cite{lattimore2020bandit} $Q^{\pi}(s,a) = Q_s^{\pi}(a)$ where $s$ is the context. While standard architectures are not designed to model such a relationship, Hypernetworks were explicitly constructed for that purpose \cite{ha2016hypernetworks}. Hypernetworks, also called meta-networks, can represent hierarchies by transforming a {\em meta} variable into a context-dependent function that maps a {\em base} variable to the required output space. This emphasizes the underlying dynamic between the meta and base variables and has found success in a variety of domains such as Bayesian neural-networks \cite{hyper_bayesian}, continual learning \cite{von2019continual}, generative models \cite{DBLP:journals/corr/abs-1901-11058} and adversarial defense \cite{sun2017hypernetworks}. The practical success has sparked interest in the theoretical properties of Hypernetworks. For example, it has recently been shown that they enjoy better parameter complexity than classical models which concatenate the base and meta-variables together \cite{galanti2020comparing, galanti2020modularity}.

When analyzing the critic's ability to represent the $Q$-function, it is important to notice that in order to optimize the policy, modern off-policy actor-critic algorithms \cite{fujimoto2018addressing,haarnoja2018soft} utilize only the parametric neural gradient of the critic with respect to the action input,  i.e., $\nabla_a Q_{\theta}^{\pi}(s, a)$.\footnote{This is in contrast to the REINFORCE approach \cite{williams1992simple} based on the policy gradient theorem \cite{sutton2000policy} which does not require a differentiable $Q$-function estimation.} Recently, \cite{ilyas2019closer} examined the accuracy of the policy gradient in on-policy algorithms. They demonstrated that standard RL implementations achieve gradient estimation with a near-zero cosine similarity when compared to the ``true" gradient. Therefore, recovering better gradient approximations has the potential to substantially improve the RL learning process. Motivated by the need to obtain high-quality gradient approximations, we set out to investigate the gradient accuracy of Hypernetworks with respect to standard models. In Sec. \ref{sec:recomposing_actor_critic} we analyze three critic models and find that the Hypernetwork model with a state as a meta-variable enjoys better gradient accuracy which translates into a faster learning rate. 

Much like the induced hierarchy in the critic, meta-policies that optimize multi-task RL problems have a similar structure as they combine a task-dependent context and a state input. While some algorithms like MAML \cite{finn2017model} and LEO \cite{rusu2018metalearning} do not utilize an explicit context, other works, e.g. PEARL \cite{rakelly2019efficient} or MQL \cite{fakoor2019meta}, have demonstrated that a context improves the generalization abilities. Recently, \cite{jayakumar2019multiplicative} have shown that Multiplicative Interactions (MI) are an excellent design choice when combining states and contexts. MI operations can be viewed as shallow Hypernetwork architectures. In Sec. \ref{sec:recomposing_meta}, we further explore this approach and study context-based meta-policies with \textit{deep} Hypernetworks.  We find that with Hypernetworks, the task and state-dependent gradients are disentangled s.t. the state-dependent gradients are marginalized out, which leads to an empirically lower learning step variance. This is specifically important in on-policy methods such as MAML, where there are fewer optimization steps during training.

The contributions of this paper are three-fold. First, in Sec. \ref{sec:recomposing_actor_critic} we provide a theoretical link between the $Q$-function gradient approximation quality and the allowable learning rate for monotonic policy improvement. Next, we show empirically that Hypernetworks achieve better gradient approximations which translates into a faster learning rate and improves the final performance. Finally, in Sec. \ref{sec:recomposing_meta} we show that Hypernetworks significantly reduce the learning step variance in Meta-RL. We summarize our empirical results in Sec. \ref{sec:experiments}, which demonstrates the gain of Hypernetworks both in single-task RL and Meta-RL. Importantly, we find empirically that Hypernetwork policies eliminate the need for the MAML adaptation step and improve the Out-Of-Distribution generalization in PEARL.

\section{Hypernetworks}
\label{sec:Hypernetworks}

A Hypernetwork \cite{ha2016hypernetworks} is a neural-network architecture designed to process a tuple $(z, x) \in Z \times X$ and output a value $y\in Y$. It is comprised of two networks, a {\em primary} network $w_{\theta}:Z\to \mathbb{R}^{n_w}$ which produces weights $w_{\theta}(z)$ for a {\em dynamic} network $f_{w_{\theta}(z)}:X\to Y$. Both networks are trained together, and the gradient flows through $f$ to the primary networks' weights $\theta$. During test time or inference, the primary weights are fixed while the $z$ input determines the dynamic network's weights.

The idea of learnable context-dependent weights can be traced back to \cite{mcclelland1985putting,schmidhuber1992learning}. However, only in recent years have Hypernetworks gained popularity when they have been applied successfully with many dynamic network models, e.g. recurrent networks \cite{ha2016hypernetworks}, MLP networks for 3D point clouds  \cite{littwin2019deep}, spatial transformation \cite{potapov2018hypernets}, convolutional networks for video frame prediction \cite{jia2016dynamic} and few-shot learning \cite{brock2018smash}. In the context of RL, Hypernetworks were also applied, e.g., in QMIX \cite{rashid2018qmix} to solve Multi-agent RL tasks and for continual model-based RL \cite{huang2020continual}. 

Fig. \ref{fig:HyperSceme} illustrates our Hypernetwork model. The primary network $w_{\theta}(z)$ contains residual blocks \cite{srivastava2015training} which transform the meta-variable into a 1024 sized latent representation. This stage is followed by a series of parallel linear transformations, termed ``heads", which output the sets of dynamic weights. The dynamic network $f_{w_{\theta}(z)}(x)$ contains only a single hidden layer of 256 which is smaller than the standard MLP architecture used in many RL papers \cite{fujimoto2018addressing,haarnoja2018soft} of 2 hidden layers, each with 256 neurons. The computational model of each dynamic layer is
\begin{equation}
    x^{l+1} = \sigma_{ReLU}\left((1 + g^l(z)) \odot x^l W^l(z) + b^l(z) \right)
\end{equation}
where the non-linearity is applied only over the hidden layer and $g^l$ is an additional gain parameter that is required in Hypernetwork architectures \cite{littwin2019deep}. We defer the discussion of these design choices to Sec. \ref{sec:experiments}.

\section{Recomposing the Actor-Critic's $Q$-Function}
\label{sec:recomposing_actor_critic}

\subsection{Background}

Reinforcement Learning concerns finding optimal policies in Markov Decision Processes (MDPs). An MDP \cite{dean1997model} is defined by a tuple $(\mathcal{S},\mathcal{A},\mathcal{P},R)$ where $\mathcal{S}$ is a set of states, $\mathcal{A}$ is a set of actions, $\mathcal{P}$ is a set of probabilities to switch from a state $s$ to $s'$ given an action $a$, and $R:\mathcal{S}\times\mathcal{A}\to\mathbb{R}$ is a scalar reward function. The objective is to maximize the expected discounted sum of rewards with a discount factor $\gamma > 0$
\begin{equation}
\label{eq:mdp_objective}
    J(\pi) = \mathbb{E}\left[\sum_{t=0}^{\infty} \gamma^t R(s_t, a_t) \middle | a_t\sim\pi(\cdot|s_t) \right].
\end{equation}
$J(\pi)$ can also be written, up to a constant factor $1-\gamma$, as an expectation over the $Q$-function
\begin{equation}
\label{eq:mdp_objective_q}
    J(\pi) = \mathbb{E}_{s\sim d^{\pi}}\left[\mathbb{E}_{a\sim \pi(\cdot|s)}\left[ Q^{\pi}(s,a)\right]\right],
\end{equation}
where the $Q$-function is the expected discounted sum of rewards following visitation at state $s$ and execution of action $a$ \cite{sutton2018reinforcement}, and $d^{\pi}$ is the state distribution induced by policy $\pi$.

Actor-critic methods maximize $J(\pi)$ over the space of parameterized policies. Stochastic policies are constructed as a state dependent transformation of an independent random variable
\begin{equation}
    \pi_{\phi}(a|s) = \mu_{\phi}(\varepsilon| s) \text{s.t.} \varepsilon\sim p_{\varepsilon},
\end{equation}
where $p_{\varepsilon}$ is a predefined multivariate distribution over $\mathbb{R}^{n_a}$ and $n_a$ is the number of actions.\footnote{Deterministic policies, on the other hand, are commonly defined as a deterministic transformation of the state's feature vector.} To maximize $J(\pi_{\phi})$ over the $\phi$ parameters, actor-critic methods operate with an iterative three-phase algorithm. First, they collect into a replay buffer $\mathcal{D}$ the experience tuples $(s,a,r,s')$ generated with the parametric $\pi_{\phi}$ and some additive exploration noise policy \cite{zhang2017deeper}. Then they fit a critic which is a parametric model $Q_{\theta}^{\pi}$ for the $Q$-function. For that purpose, they apply TD-learning \cite{sutton2018reinforcement} with the loss function
\begin{multline*}
    \mathcal{L}_{critic}(\theta) = \\
    \mathbb{E}_{s,a,r,s'\sim\mathcal{D}}\left[\left|Q^{\pi}_{\theta}(s,a) - r - \gamma \mathbb{E}_{a'\sim\pi_{\phi}(\cdot|s')}[Q^{\pi}_{\bar{\theta}}(s',a')]\right|^2\right],
\end{multline*}
where $\bar{\theta}$ is a lagging set of parameters \cite{lillicrap2016continuous}. Finally, they apply gradient descent updates in the direction of an off-policy surrogate of $J(\pi_{\phi})$
\begin{equation}
\label{eq:j_actor}
\begin{aligned}
    & \phi\leftarrow \phi + \eta \nabla_{\phi} J_{actor}(\phi) \\
    & \nabla_{\phi} J_{actor}(\phi) = \mathbb{E}_{\bfrac{s \sim\mathcal{D}}{\varepsilon \sim p_{\varepsilon}}}\left[\nabla_{\phi} \mu_{\phi}(\varepsilon| s) \nabla_a Q^{\pi}_{\theta}(s,\mu_{\phi}(\varepsilon| s))\right].
\end{aligned}
\end{equation}
Here, $\nabla_{\phi} \mu_{\phi}(\varepsilon| s)$ is a matrix of size $n_{\phi}\times n_a$ where $n_{\phi}$ is the number of policy parameters to be optimized.

Two well-known off-policy algorithms are TD3 \cite{fujimoto2018addressing} and SAC \cite{haarnoja2018soft}. TD3 optimizes deterministic policies with additive normal exploration noise and double $Q$-learning to improve the robustness of the critic part \cite{hasselt2016deep}. On the other hand, SAC adopts stochastic, normally distributed policies but it modifies the reward function to include a high entropy bonus $\tilde{R}(s,a) = R(s,a) + \alpha H(\pi(\cdot|s))$ which eliminates the need for exploration noise.

\subsection{Our Approach}

The gradient of the off-policy surrogate $\nabla_{\phi} J_{actor}(\phi)$ differs from the true gradient $\nabla_{\phi} J(\pi)$ in two elements: First, the distribution of states is the empirical distribution in the dataset and not the policy distribution $d^{\pi}$; and second, the $Q$-function gradient is estimated with the critic's parametric neural gradient $\nabla_{a} Q^{\pi}_{\theta} \simeq \nabla_{a} Q^{\pi}$. Avoiding a distribution mismatch is the motivation of many constrained policy improvement methods such as TRPO and PPO \cite{schulman2015trust, schulman2017proximal}. However, it requires very small and impractical steps. Thus, many off-policy algorithms ignore the distribution mismatch and seek to maximize only the empirical advantage
\begin{equation*}
    A(\phi',\phi) = \mathbb{E}_{s\sim\mathcal{D}}\left[\mathbb{E}_{a\sim\pi'}\left[Q^{\pi}(s,a)\right] - \mathbb{E}_{a\sim\pi}\left[Q^{\pi}(s,a)\right]\right].
\end{equation*}
In practice, a positive empirical advantage is associated with better policies and is required by monotonic policy improvement methods such as TRPO \cite{kakade2002approximately,schulman2015trust}. Yet, finding positive empirical advantage policies requires a good approximation of the gradient $\nabla_{a} Q^{\pi}$. The next proposition suggests that with a sufficiently accurate approximation, applying the gradient step as formulated in the actor update in Eq. (\ref{eq:j_actor}) yields positive empirical advantage policies.

\begin{prop}
Let $\pi(a|s) = \mu_{\phi}(\varepsilon|s)$ be a stochastic parametric policy with $\varepsilon \sim p_{\varepsilon}$, and $\mu_{\phi}(\cdot|s)$ a transformation with a Lipschitz continuous gradient and a Lipschitz constant $\kappa_{\mu}$. Assume that its $Q$-function $Q^{\pi}(s,a)$ has a Lipschitz continuous gradient in $a$, i.e. $|\nabla_a Q^{\pi}(s,a_1) - \nabla_a Q^{\pi}(s,a_2)| \leq \kappa_q \|a_1 - a_2\|$. Define the average gradient operator $\overline{\nabla}_{\phi} \cdot f = \mathbb{E}_{s\sim\mathcal{D}}\left[\mathbb{E}_{\varepsilon\sim p_{\varepsilon}}[\nabla_{\phi} \mu_{\phi}(\varepsilon| s) \cdot f(s, \mu_{\phi}(\varepsilon|s))]\right]$. If there exists a gradient estimation $g(s,a)$ and $0 < \alpha < 1$ s.t.
\begin{equation}
\label{eq:alpha_rmse}
    \| \overline{\nabla}_{\phi}\cdot g - \overline{\nabla}_{\phi} \cdot \nabla_{a} Q^{\pi} \| \leq \alpha \| \overline{\nabla}_{\phi} \cdot \nabla_{a} Q^{\pi} \|
\end{equation}
then the ascent step $\phi' \leftarrow \phi + \eta \overline{\nabla}_{\phi}\cdot g$ with $\eta \leq \frac{1}{\tilde{k}}\frac{1-\alpha}{(1 + \alpha)^2}$ yields a positive empirical advantage policy.
\end{prop}

We define $\tilde{k}$ and provide the proof in the appendix. It follows that a positive empirical advantage can be guaranteed when the gradient of the $Q$-function is sufficiently accurate, and with better gradient models, i.e. smaller $\alpha$, one may apply larger ascent steps. However, instead of fitting the gradient, actor-critic algorithms favor modeling the $Q$-function and estimate the gradient with the parametric gradient of the model $\nabla_a Q^{\pi}_{\theta}$. It is not obvious whether better models for the $Q$-functions, with lower Mean-Squared-Error (MSE), provide better gradient estimation. A more direct approach could be to explicitly learn the gradient of the $Q$-function \cite{sarafian2020explicit,saremi2019approximating}; however, in this work, we choose to explore which architecture recovers more accurate gradient approximation based on the parametric gradient of the $Q$-function model. 

We consider three alternative models:
\begin{enumerate}
    \item MLP network, where state features $\xi(s)$ (possibly learnable) are concatenated into a single input of a multi-layer linear network.
    \item Action-State Hypernetwork (AS-Hyper) where the actions are the {\em meta} variable, input of the primary network $w$, and the state features are the {\em base} variable, input for the dynamic network $f$.
    \item State-Action Hypernetwork (SA-Hyper), which reverses the order of AS-Hyper.
\end{enumerate}

\vspace{-.4cm}
\begin{figure}[!ht]
\begin{center}
    \includegraphics[width=\linewidth]{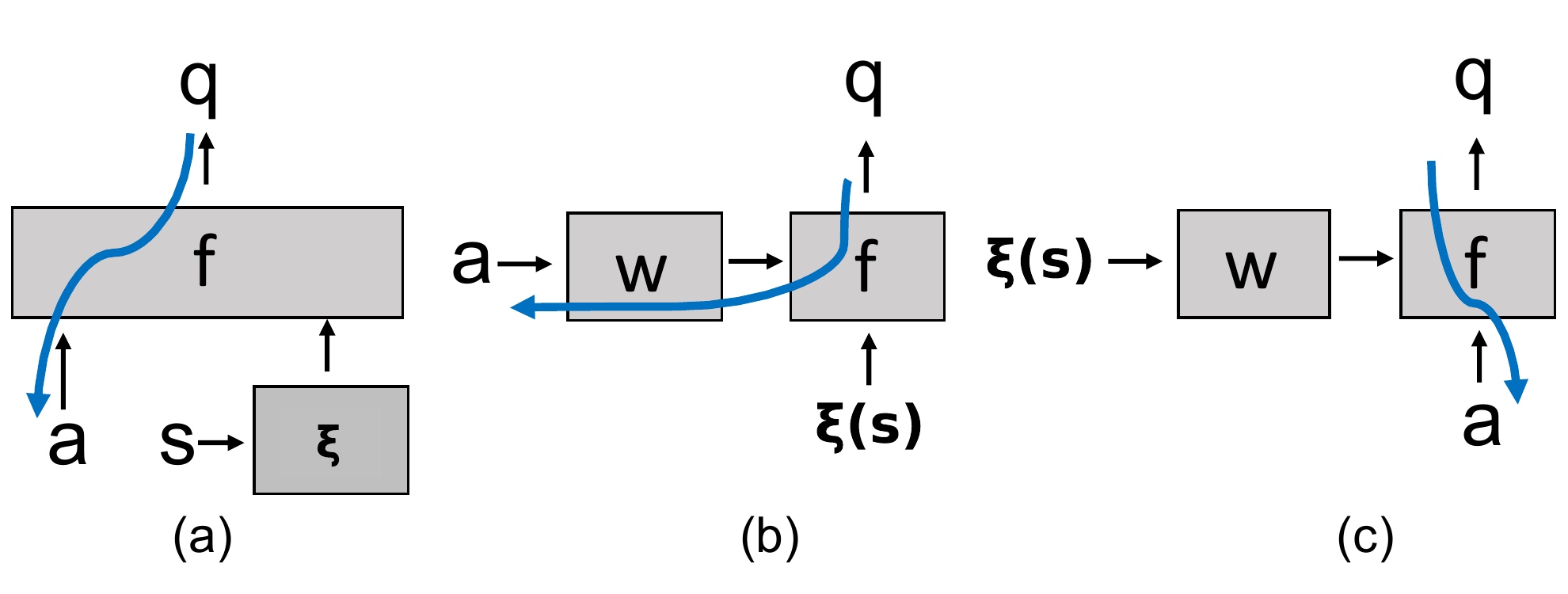}
\end{center}
\caption{Illustrating three alternatives for combining states and actions: (a) MLP; (b) AS-Hyper; and (c) SA-Hyper. The blue arrows represent the backpropagation calculation of the actions gradient. Notice that in the SA-Hyper, the gradient flows only through the dynamic network, which enables more efficient implementation as the dynamic network is much smaller than the primary network.}
\label{fig:sa_composition_alternatives}
\vspace{-.1cm}
\end{figure}

To develop some intuition, let us first consider the simplest case where the dynamic network has a single linear layer and the MLP model is replaced with a plain linear model. Starting with the linear model, the $Q$-function and its gradient take the following parametric model:
\begin{equation}
    \begin{aligned}
        Q^{\pi}_{\theta}(s,a) &= [w_s, w_{a}] \cdot [\xi(s), a] \\
        \nabla_a Q^{\pi}_{\theta}(s,a) &= w_a
    \end{aligned}
\end{equation}
where $\theta=[w_s, w_a]$. Clearly, in this case, the gradient is not a function of the state, therefore it is impossible to exploit this model for actor-critic algorithms. For the AS-Hyper we obtain the following model
\begin{equation}
    \begin{aligned}
        Q^{\pi}_{\theta}(s,a) &= w(a) \cdot \xi(s) \\
        \nabla_a Q^{\pi}_{\theta}(s,a) &= \nabla_a w(a) \xi(s)
    \end{aligned}
\end{equation}
Usually, the state feature vector $\xi(s)$ has a much larger dimension than the action dimension $n_a$. Thus, the matrix $\nabla_a w(a)$ has a large null-space which can potentially hamper the training as it may yield zero or near-zero gradients even when the true gradient exists. 

\begin{figure*}[!ht]
\begin{center}
    \includegraphics[width=\linewidth]{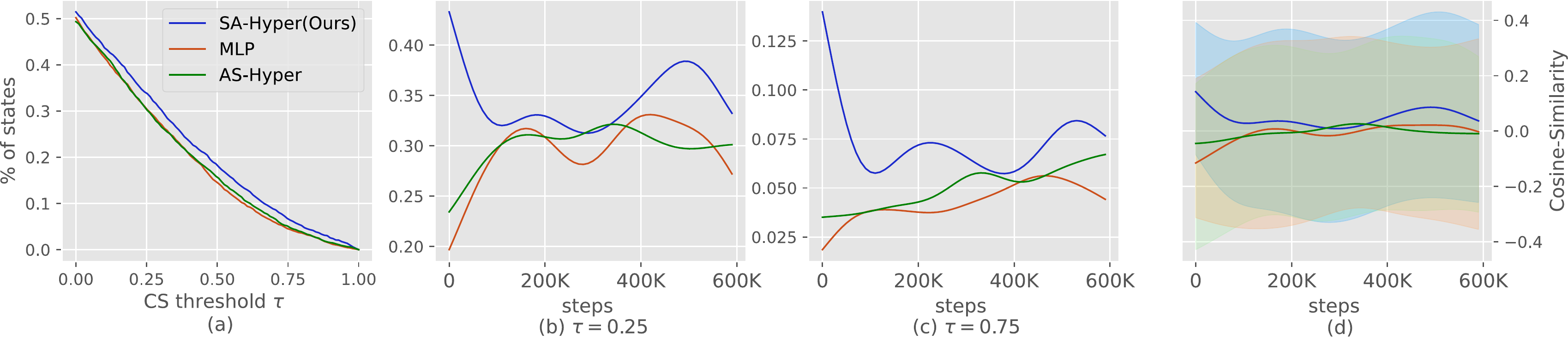}
\end{center}
\vspace{-.3cm}
\caption{Comparing the Cosine-Similarity of different critic models: (a) The percentage of states with CS better than a $\tau$ threshold. (b-c) The percentage of states with CS better than $\tau=0.25$ and $\tau=0.75$ with respect to the learning step. (d) The mean CS over time averaged over all seeds and environments. The shaded area is the interquartile range $Q3-Q1$. In all cases, the CS was evaluated every $10K$ steps with $N_s=15$ states and $N_r=15$ independent trajectories for each state.}
\label{fig:cs_10k}
\vspace{-.2cm}
\end{figure*}

On the other hand, the SA-Hyper formulation is
\begin{equation}
    \begin{aligned}
        Q^{\pi}_{\theta}(s,a) &= w(s) \cdot a \\
        \nabla_a Q^{\pi}_{\theta}(s,a) &= w(s)
    \end{aligned}
\end{equation}
which is a state-dependent constant model of the gradient in $a$. While it is a relatively naive model, it is sufficient for localized policies with low variance as it approximates the tangent hyperplane around the policy mean value.

Moving forward to a multi-layer architecture, let us first consider the AS-Hyper architecture. In this case the gradient is $\nabla_a Q^{\pi}_{\theta}(s,a) = \nabla_a w(a) \nabla_w f_{w}(s) $. We see that the problem of the single layer is exacerbated since $\nabla_a w(a)$ is now a $n_a \times n_w$ matrix where $n_w \gg n_a$ is the number of dynamic network weights. 

Next, the MLP and SA-Hyper models can be jointly analyzed. First, we calculate the input's gradient of each layer
\begin{align}
    x^{l+1} &= f^l(x^l) = \sigma\left(x^l W^l + b^l\right) \\
    \nabla_a x^{l+1} &= (\nabla_a x^{l}) \nabla_{x^l} f^l(x^l) = (\nabla_a x^{l})  W^l \Lambda^l(x^l) \\
    \Lambda^l(x^l) &= \diag\left(\sigma'\left(x^l W^l + b^l\right)\right),
\end{align}
where $\sigma$ is the activation function and $W^l$ and $b^l$ are the weights and biases of the $l$-th layer, respectively. By the chain rule, the input's gradient of an $L$-layers network is the product of these expressions. For the MLP model we obtain
\begin{equation}
    \nabla_a Q^{\pi}_{\theta}(s,a) = W^a \Lambda^1(s, a) \left(\prod_{l=2}^{L-1} W^l \Lambda^l(s, a)\right) W^L.
\end{equation}
On the other hand, in SA-Hyper the weights are the outputs of the primary network, thus we have
\begin{multline}
    \nabla_a Q^{\pi}_{\theta}(s,a) = \\ 
    W^1(s) \Lambda^1(s, a) \left(\prod_{l=2}^{L-1} W^l(s) \Lambda^l(s, a)\right) W^L(s) .
\end{multline}
Importantly, while the SA-Hyper's gradient configuration is controlled via the state-dependent matrices $W^l(s)$, in the MLP model, it is a function of the state only via the diagonal elements in $\Lambda^l(s, a)$. These local derivatives of the non-linear activation functions are usually piecewise constant when the activations take the form of ReLU-like functions. Also, they are required to be bounded and smaller than one in order to avoid exploding gradients during training \cite{philipp2017exploding}. These restrictions significantly reduce the expressiveness of the parametric gradient and its ability to model the true $Q$-function gradient. For example, with ReLU, for two different pairs $(s_1,a_1)$ and $(s_2,a_2)$ the estimated gradient is equal if they have same active neurons map (i.e. the same ReLUs are in the active mode). Following this line of reasoning, we postulate that the SA-Hyper configuration should have better gradient approximations.

\begin{figure*}[!ht]
\begin{center}
    \includegraphics[width=\linewidth]{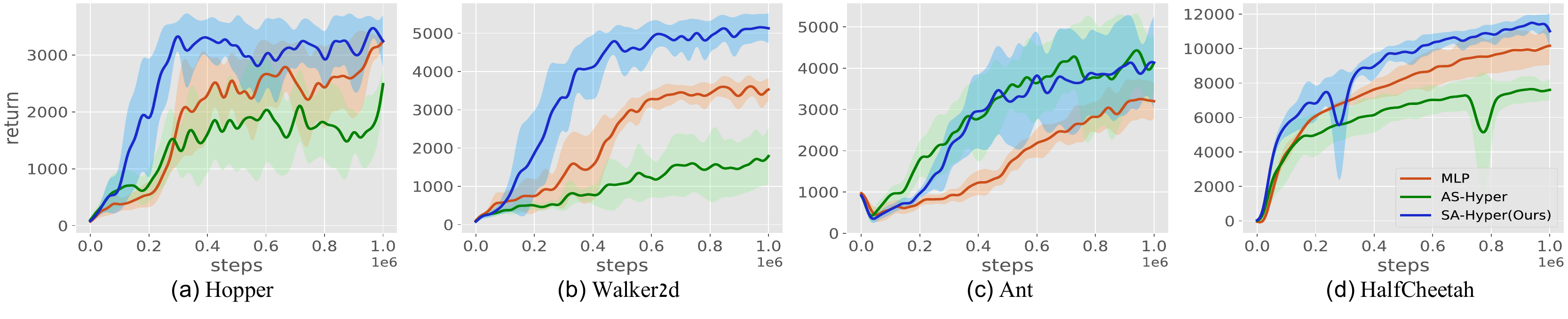}
\end{center}
\caption{Learning curves of the TD3 algorithm with different critic models. SA-Hyper refers to $Q^\pi_\theta=f_{w_\theta(s)}(a)$, AS-Hyper refers to $Q^\pi_\theta=f_{w_\theta(a)}(s)$ and MLP refers to  $Q^\pi_\theta = f_{\theta}(s,a)$, which concatenates both inputs.}
\label{fig:td3_order}
\vspace{-0.2cm}
\end{figure*}

\par{\bf Empirical analysis} To test our hypothesis, we trained TD3 agents with different network models and evaluated their parametric gradient $\nabla_a Q_{\theta}(s, a)$. To empirically analyze the gradient accuracy, we opted to estimate the true $Q$-function gradient with a non-parametric local estimator at the policy mean value, i.e. at $a_\mu=\mathbb{E}_{\varepsilon\sim p_{\varepsilon}} \left[\mu_{\phi}(\varepsilon| s)\right]$. For that purpose, we generated $N_r$ independent trajectories with actions sampled around the mean value, i.e. $a = a_\mu + \Delta_a$, and fit with a Least-Mean-Square (LMS) estimator a linear model for the empirical return of the sampled trajectories. The ``true" gradient is therefore the linear model's gradient. Additional technical details of this estimator are found in the appendix.

As our $Q$-function estimator is based on Temporal-Difference (TD) learning, it bears bias. Hence, in practice we cannot hope to reconstruct the true $Q$-function scale. Thus, instead of evaluating the gradient's MSE, we take the Cosine Similarity (CS) as a surrogate for measuring the gradient accuracy.
\begin{equation*}
    cs(Q^{\pi}_{\theta}) = \mathbb{E}_{s\sim\mathcal{D}}\left[\frac{\nabla_a Q_{\theta}^{\pi}(s,a_{\mu}) \cdot \nabla_a Q^{\pi}(s,a_{\mu})}{\|\nabla_a Q_{\theta}^{\pi}(s,a_{\mu})\| \ \| \nabla_a Q^{\pi}(s,a_{\mu})\|}\right],
\end{equation*}
Fig. \ref{fig:cs_10k} summarizes our CS evaluations with the three model alternatives averaged over 4 Mujoco \cite{todorov2012mujoco} environments. Fig. \ref{fig:cs_10k}d presents the mean CS over states during the training process. Generally, the CS is very low, which indicates that the RL training is far from optimal. While this finding is somewhat surprising, it corroborates the results in \cite{ilyas2019closer} which found near-zero CS in policy gradient algorithms. Nevertheless, note that the impact of the CS accuracy is cumulative as in each gradient ascent step the policy accumulates small improvements. This lets even near-zero gradient models improve over time. Overall, we find that the SA-Hyper CS is higher, and unlike other models, it is larger than zero during the entire training process. The SA-Hyper advantage is specifically significant at the first $100K$ learning steps, which indicates that SA-Hyper learns faster in the early learning stages. 

Assessing the gradient accuracy by the average CS can be somewhat confounded by states that have reached a local equilibrium during the training process. In these states the true gradient has zero magnitude s.t. the CS is ill-defined. For that purpose, in Fig. \ref{fig:cs_10k}a-c we measure the percentage of states with a CS higher than a threshold $\tau$. This indicates how many states are \textit{learnable} where more learnable states are attributed to a better gradient estimation. Fig. \ref{fig:cs_10k}a shows that for all thresholds $\tau \in [0,1]$ SA-Hyper has more learnable states, and Fig. \ref{fig:cs_10k}b-c present the change in learnable states for different $\tau$ during the training process. Here we also find that the SA-Hyper advantage is significant particularly at the first stage of training. Finally, Fig. \ref{fig:td3_order} demonstrates how gradient accuracy translates to better learning curves. As expected, we find that SA-Hyper outperforms both the MLP architecture and the AS-Hyper configuration which is also generally inferior to MLP.

In the next section, we discuss the application of Hypernetworks in Meta-RL for modeling context conditional policies. When such a context exists, it also serves as an input variable to the $Q$-function. In that case, when modeling the critic with a Hypernetwork, one may choose to use the context as a meta-variable or alternatively as a base variable. Importantly, when the context is the dynamic's input, the dynamic weights are fixed for each state, regardless of the task. In our PEARL experiments in Sec. \ref{sec:experiments} we always used the context as a base variable of the critic. We opted for this configuration since: (1) we found empirically that it is important for the generalization to have a constant set of weights for each state; and (2) As the PEARL context is learnable, we found that when the context gradient backpropagates through three networks (primary, dynamic and the context network), it hampers the training. Instead, as a base variable, the context's gradient backpropagates only via two networks as in the original PEARL implementation.

\begin{figure*}[!ht]
\begin{center}
    \includegraphics[width=1.\linewidth]{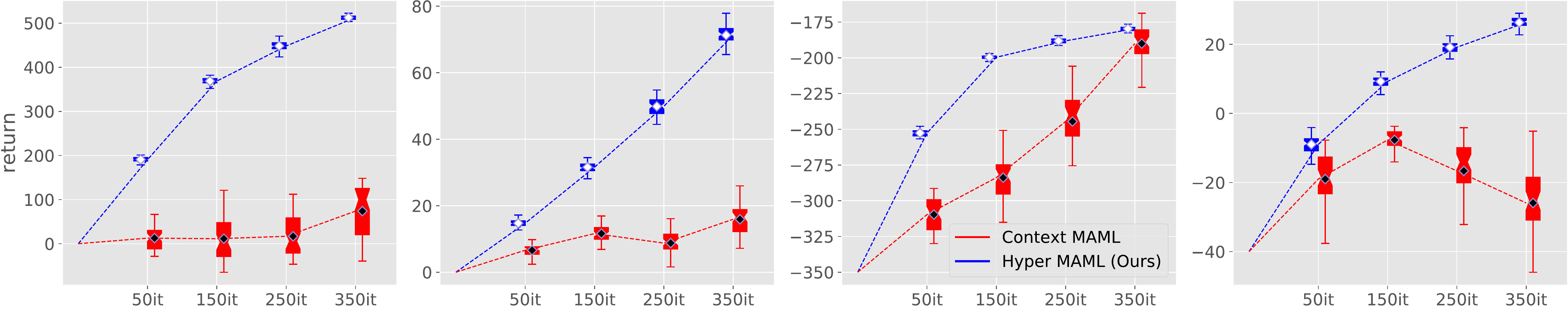}
\end{center}
\vspace{-0.2cm}
\caption{Visualizing gradient noise in MAML: The statistical population of the performance after 50 uncorrelated update steps is plotted for 4 different time steps. Hyper-MAML refers to Hypernetwork where the oracle-context is the meta variable and the state features are the base variable. Context-MAML refers to MLP policy where the oracle-context is concatenated with the state features.}
\label{fig:Meta_norm}
\vspace{-0.2cm}
\end{figure*}

\section{Recomposing the Policy in Meta-RL}
\label{sec:recomposing_meta}

\subsection{Background}

Meta-RL is the generalization of Meta-Learning \cite{mishra2018simple,sohn2019meta} to the RL domain. It aims at learning meta-policies that solve a distribution of different tasks $p(\mathcal{T})$. Instead of learning different policies for each task, the meta-policy shares weights between all tasks and thus can generalize from one task to the other \cite{sung2017learning}. A popular Meta-RL algorithm is MAML \cite{finn2017model}, which learns a set of weights that can quickly adapt to a new task with a few gradient ascent steps. To do so, for each task, it estimates the policy gradient \cite{sutton2000policy} at the adaptation point. The total gradient is the sum of policy gradients over the task distribution
$p(\mathcal{T})$:
\begin{equation}
\label{eq:maml}
\begin{aligned}
\nabla_{\phi} J_{maml}(\phi) &=  \mathbb{E}_{\bfrac{\mathcal{T}_i\sim p(\mathcal{T})}{\pi_{\phi_i}}}\left[\sum_{t=0}^{\infty}\hat{A}_{i,t}\nabla_{\phi} \log\pi_{\phi_i}(a_t|s_t)\right] \\
\phi_i &= \phi + \eta\mathbb{E}_{\pi_{\phi}}\left[\sum_{t=0}^{\infty}\hat{A}_{i,t}\nabla_{\phi}\log\pi_{\phi_i}(a_t|s_t)\right],
\end{aligned}
\end{equation}
where $\hat{A}_{i,t}$ is the empirical advantage estimation at the $t$-th step in task $i$ \cite{schulman2015high}. On-policy algorithms tend to suffer from high sample complexity as each update step requires many new trajectories sampled from the most recent policy in order to adequately evaluate the gradient direction. 

Off-policy methods are designed to improve the sample complexity by reusing experience from old policies \cite{thomas2016data}. Although not necessarily related, in Meta-RL, many off-policy algorithms also avoid the MAML approach of weight adaptation. Instead, they opt to condition the policy and the $Q$-function on a context which distinguishes between different tasks \cite{rencontext,sung2017learning}. A notable off-policy Meta-RL method is PEARL \cite{rakelly2019efficient}. It builds on top of the SAC algorithm and learns a $Q$-function $Q_{\theta}^{\pi}(s, a, z)$, a policy $\pi_{\phi}(s, z)$ and a context $z\sim q_{\nu}(z|c^{\mathcal{T}_i})$. The context, which is a latent representation of task $\mathcal{T}_i$, is generated by a probabilistic model that processes a trajectory $c^{\mathcal{T}_i}$ of $(s,a,r)$ transitions sampled from task $\mathcal{T}_i$. To learn the critic alongside the context, PEARL modifies the SAC critic loss to
\begin{multline*}
    \mathcal{L}_{pearl}^{critic}(\theta,\nu) = \\ \mathbb{E}_{\mathcal{T}}\left[\mathbb{E}_{q_{\nu}(z|c^{\mathcal{T}_i})}\left[\mathcal{L}_{sac}^{critic}(\theta,\nu) + D_{KL}\left(q_{\nu}(z|c^{\mathcal{T}_i})\middle| p(z)\right)\right]\right],
\end{multline*}
where $p(z)$ is a prior probability over the latent distribution of the context. While PEARL's context is a probabilistic model, other works \cite{fakoor2019meta} have suggested that a deterministic learnable context can provide similar results.

In this work, we consider both a learnable context and also the simpler approach of an {\em oracle-context} $c^{\mathcal{T}_i}$ which is a unique, predefined identifier for task $i$ \cite{jayakumar2019multiplicative}. It can be an index when there is a countable number of tasks or a continuous number when the tasks are sampled from a continuous distribution. In practice, the oracle identifier is often known to the agent. Moreover, sometimes, e.g., in goal-oriented tasks, the context cannot be recovered directly from the transition tuples without prior knowledge, since there are no rewards unless the goal is reached, which rarely happens without policy adaptation.

\subsection{Our Approach}

Hypernetworks naturally fit into the meta-learning formulation where the context is an input to the primary network \cite{von2019continual,zhao2020meta}. Therefore, we suggest modeling meta-policies s.t. the context is the meta variable and the state is the dynamic's input
\begin{equation}
\label{eq:meta_policy}
    \pi_{\phi}(a|s, c) = \mu_{w(c)}(\varepsilon|s) \ \ \text{s.t.} \ \ \varepsilon\sim p_{\varepsilon}.
\end{equation}

Interestingly, this modeling disentangles the state dependent gradient and the task dependent gradient of the meta-policy. To see that, let us take for example the on-policy objective of MAML and plug in a context dependent policy $\pi_{\phi}(a|s, c) = \mu_{\phi}(\varepsilon| s, c)$. Then, the objective in Eq. (\ref{eq:maml}) becomes
\begin{equation}
\label{eq:mlp_meta_objective}
J(\phi) = \sum_{\mathcal{T}_i} \sum_{s_j\in \mathcal{T}_i}\hat{A}_{i,j} \frac{\nabla_{\phi} \mu_{\phi_i}(\varepsilon_j| s_j, c_i)}{\mu_{\phi_i}(\varepsilon_j| s_j, c_i)}.
\end{equation}
Applying the Hypernetwork modeling of the meta-policy in Eq. (\ref{eq:meta_policy}), this objective can be written as
\begin{equation}
J(\phi) =
\sum_{\mathcal{T}_i} \nabla_{\phi}w(c_i) \cdot \sum_{s_j\in \mathcal{T}_i}\hat{A}_{i,j} \frac{\nabla_{w} \mu_{w(c_i)}(\varepsilon_j| s_j)}{\mu_{w(c_i)}(\varepsilon_j| s_j)}
\end{equation}
In this form, the state-dependent gradients of the dynamic weights $\nabla_{w} \mu_{w(c_i)}(\varepsilon_j, s_j)$ are averaged independently for each task, and the task-dependent gradients of the primary weights $\nabla_{\phi} w(c_i)$ are averaged only over the task distribution and not over the joint task-state distribution as in Eq. (\ref{eq:mlp_meta_objective}). We postulate that such disentanglement reduces the gradient noise for the same number of samples. This should translate to more accurate learning steps and thus a more efficient learning process.

To test our hypothesis, we trained two different meta-policy models based on the MAML algorithm: (1) an MLP model where a state and an oracle-context are joined together; and (2) a Hypernetwork model, as described, with an oracle-context as a meta-variable. Importantly, note that, other than the neural architecture, both algorithms are {\em identical}. For four different timestamps during the learning process, we constructed 50 different uncorrelated gradients from different episodes and evaluating the updated policy's performance. We take the performance statistics of the updated policies as a surrogate for the gradient noise. In Fig. \ref{fig:Meta_norm}, we plot the performance statistics of the updated meta-policies. We find that the variance of the Hypernetwork model is significantly lower than the MLP model across all tasks and environments. This indicates more efficient improvement and therefore we also observe that the mean value is consistently higher.

\section{Experiments}
\label{sec:experiments}

\subsection{Experimental Setup}
We conducted our experiments in the MuJoCo simulator \cite{todorov2012mujoco} and tested the algorithms on the benchmark environments available in OpenAI Gym \cite{openAI}. For single task RL, we evaluated our method on the: (1) Hooper-v2; (2) Walker2D-v2; (3) Ant-v2\footnotemark; and (4) Half-Cheetah-v2 environments. For meta-RL, we evaluated our method on the: (1) Half-Cheetah-Fwd-Back and (2) Ant-Fwd-Back, and on velocity tasks: (3) Half-Cheetah-Vel and (4) Ant-Vel as is done in \cite{rakelly2019efficient}. We also added the Half-Cheetah-Vel-Medium environment as presented in \cite{fakoor2019meta}, which tests out-of-distribution generalization abilities. For Context-MAML and Hyper-MAML we adopted the {\em oracle-context} as discussed in Sec. \ref{sec:recomposing_meta}. For the forward-backward tasks, we provided a binary indicator, and for the velocity tasks, we adopted a continuous context in the range $[0,3]$ that maps to the velocities in the training distribution.

In the RL experiments, we compared our model to SAC and TD3, and in Meta-RL, we compared to MAML and PEARL. We used the authors' official implementations (or open-source PyTorch \cite{ketkar2017introduction} implementation when the official one was not available) and the original baselines' hyperparameters, as well as strictly following each algorithm evaluation procedure. The Hypernetwork training was executed with the baseline loss s.t. we changed only the networks model and adjusted the learning rate to fit the different architecture. All experiments were averaged over 5 seeds. Further technical details are in the appendix.

 \subsection{The Hypernetwork Architecture}
 \label{subsec:The Hypernetwork Architecture}
Our Hypernetwork model is illustrated in Fig. \ref{fig:HyperSceme} and in Sec. \ref{sec:Hypernetworks}. When designing the Hypernetwork model, we did not search for the best performance model, rather we sought a proper comparison to the standard MLP architecture used in RL (denoted here as MLP-Standard). To that end, we used a smaller dynamic network than the MLP model (single hidden layer instead of two layers and the same number of neurons (256) in a layer). With this approach, we wish to show the gain of using dynamic weights with respect to a fixed set of weights in the MLP model. To emphasize the gain of the dynamic weights, we added an MLP-Small baseline with equal configuration to the dynamic model (one hidden layer with 256 neurons).

\begin{figure}[!ht]
\begin{center}
    \includegraphics[width=.96\linewidth]{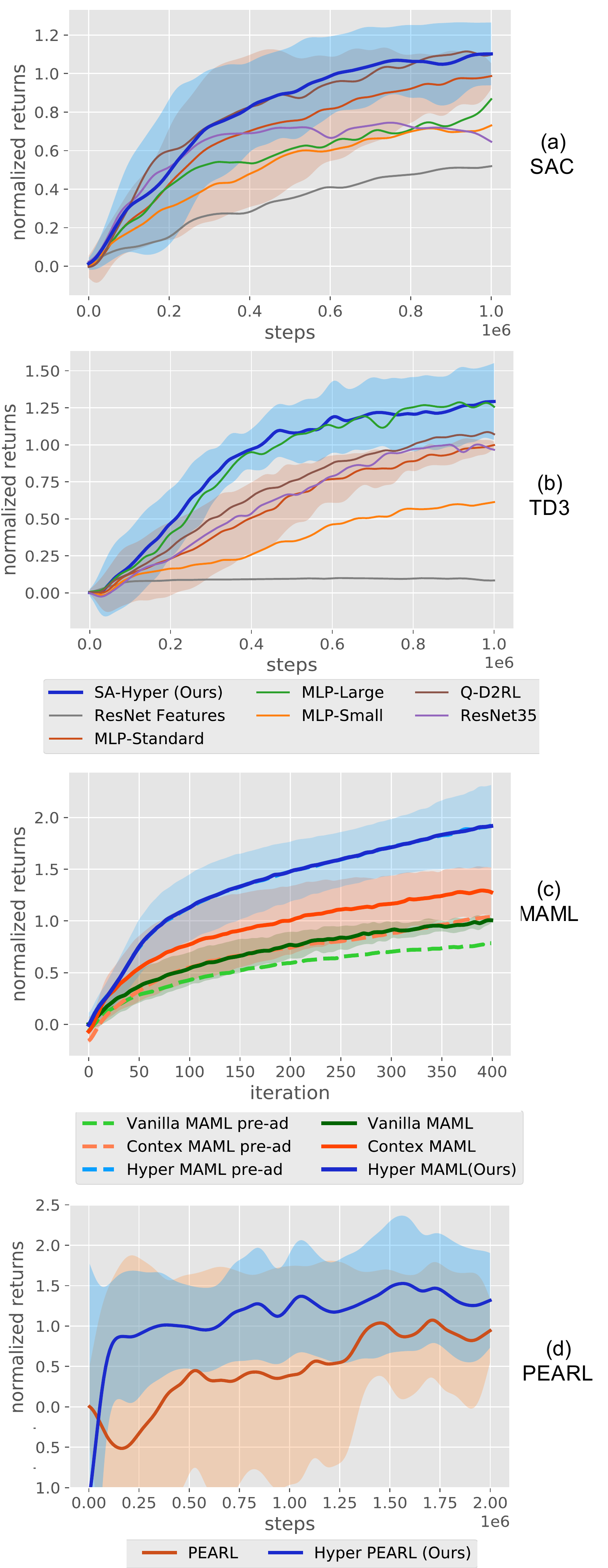}
\end{center}
\vspace{-.3cm}
\caption{The mean normalized score with respect to different baseline models: (a) SAC; (b) TD3; (c) MAML; and (d) PEARL. The Hypernetwork consistently improves all baselines in all algorithms.}
\label{fig:total_norm}
\end{figure}

Unlike the dynamic network, the role of the primary network is missing from the MLP architecture. Therefore, for the primary network, we used a high-performance ResNet model \cite{srivastava2015training} which we found apt for generating the set of dynamic weights \cite{glorot2010understanding}. To make sure that the performance gain is not due to the expressiveness of the ResNet model or the additional number of learnable weights, we added three more baselines: (1) ResNet Features: the same primary and dynamic architecture, but the output of the primary is a state feature vector which is concatenated to the action as the input for an MLP-Standard network; (2) MLP-Large: two hidden layers, each with 2900 neurons which sum up to $~9M$ weights as in the Hypernetwork architecture; and (3) Res35: ResNet with 35 blocks to yield the $Q$-value, which sum up to $~4.5M$ weights. In addition, we added a comparison to the Q-D2RL model: a deep dense architecture for the $Q$-function which was recently suggested in \cite{sinha2020d2rl}.

One important issue with Hypernetworks is their numerical stability. We found that they are specifically sensitive to weight initialization as bad primary initialization may amplify into catastrophic dynamic weights \cite{chang2019principled}. We solved this problem by initializing the primary s.t. the average initial distribution dynamic weights resembles the Kaiming-uniform initialization \cite{he2015delving}. Further details can be found in the appendix.

\footnotetext{We reduced the control cost as is done in PEARL \cite{rakelly2019efficient} to avoid numerical instability problems.}

\subsection{Results}

The results and the comparison to the baselines are summarized in Fig. \ref{fig:total_norm}. In all four experiments, our Hypernetwork model achieves an average of 10\% - 70\% gain over the MLP-Standard baseline in the final performance and reaches the baseline's score, with only 20\%-70\% of the total training steps. As described in Sec. \ref{subsec:The Hypernetwork Architecture}, for the RL experiments, in addition to the MLP-Standard model, we tested five more baselines: (1) MLP-Large; (2) MLP-Small; (3) ResNet Features; (4) ResNet35; and (5) Q-D2RL. Both on TD3 and SAC, we find a consistent improvement over all baselines and SA-Hyper outperforms in all environments with two exceptions: where MLP-Large or Q-D2RL achieve a better score than SA-Hyper in the Ant-v2 environment (the learning curves for each environment are found in the appendix). While it may seem like the Hypernetwork improvement is due to its large parametric dimension or the ResNet design of the primary model, our results provide strong evidence that this assumption is not true. The SA-Hyper model outperforms other models with the same number of parameters (MLP-Large and ResNet Features\footnote{Interestingly, The Resnet Features baseline achieved very low scores even as compared to the MLP-Standard baseline. Indeed, this result is not surprising as the action gradient model of Resnet Features is identical to the action gradient model of MLP-Small (single hidden layer with 256 neurons). While ResNet generated state features may improve the $Q$-function estimation, they do not necessarily improve the gradient estimation $\nabla_a Q^{\pi}$ as the network is not explicitly trained to model the gradient.}) and also models that employ ResNet architectures (ResNet Features and Res35). In addition, it is as good (SAC) or better (TD3) than Q-D2RL, which was recently suggested as an architecture tailored for the RL problem \cite{sinha2020d2rl}. Please note that as discussed in Sec.  \ref{subsec:The Hypernetwork Architecture} and unlike D2RL, we do not optimize the number of layers in the dynamic model.\footnote{We do not compare to the full D2RL model which also modifies the policy architecture as our SA-Hyper model only changes the $Q$-net model.}

In Fig. \ref{fig:total_norm}c we compared different models for MAML: (1) Vanilla-MAML; (2) Context-MAML, i.e. a context-based version of MAML with an oracle-context; and (3) Hyper-MAML, similar to context-MAML but with a Hypernetwork model. For all models, we evaluated both the pre-adaptation (pre-ad) as well as the post-adaptation scores. First, we verify the claim in \cite{fakoor2019meta} that context benefits Meta-RL algorithms just as Context-MAML outperforms Vanilla-MAML. However, we find that Hyper-MAML outperforms Context-MAML by roughly 50\%. Moreover, unlike the standard MLP models, we find that Hyper-MAML does not require any adaptation step (no observable difference between the pre- and post-adaptation scores). We assume that this result is due to the better generalization capabilities of the Hypernetwork architecture as can also be seen from the next PEARL experiments.

In Fig. \ref{fig:total_norm}d we evaluated the Hypernetwork model with the PEARL algorithm. The context is learned with a probabilistic encoder as presented in \cite{rakelly2019efficient} s.t. the only difference with the original PEARL is the policy and critic neural models. The empirical results show that Hyper-PEARL outperforms the MLP baseline both in the final performance (15\%) and in sample efficiency (70\% fewer steps to reach the final baseline score). Most importantly, we find that Hyper-PEARL generalizes better to the unseen test tasks. This applies both to test tasks sampled from the training distribution (as the higher score and lower variance of Hyper-PEARL indicate) and also to Out-Of-Distribution (OOD) tasks, as can be observed in Fig. \ref{fig:OOD-PEARL}.

\begin{figure}[!ht]
\begin{center}
    \includegraphics[width=\linewidth]{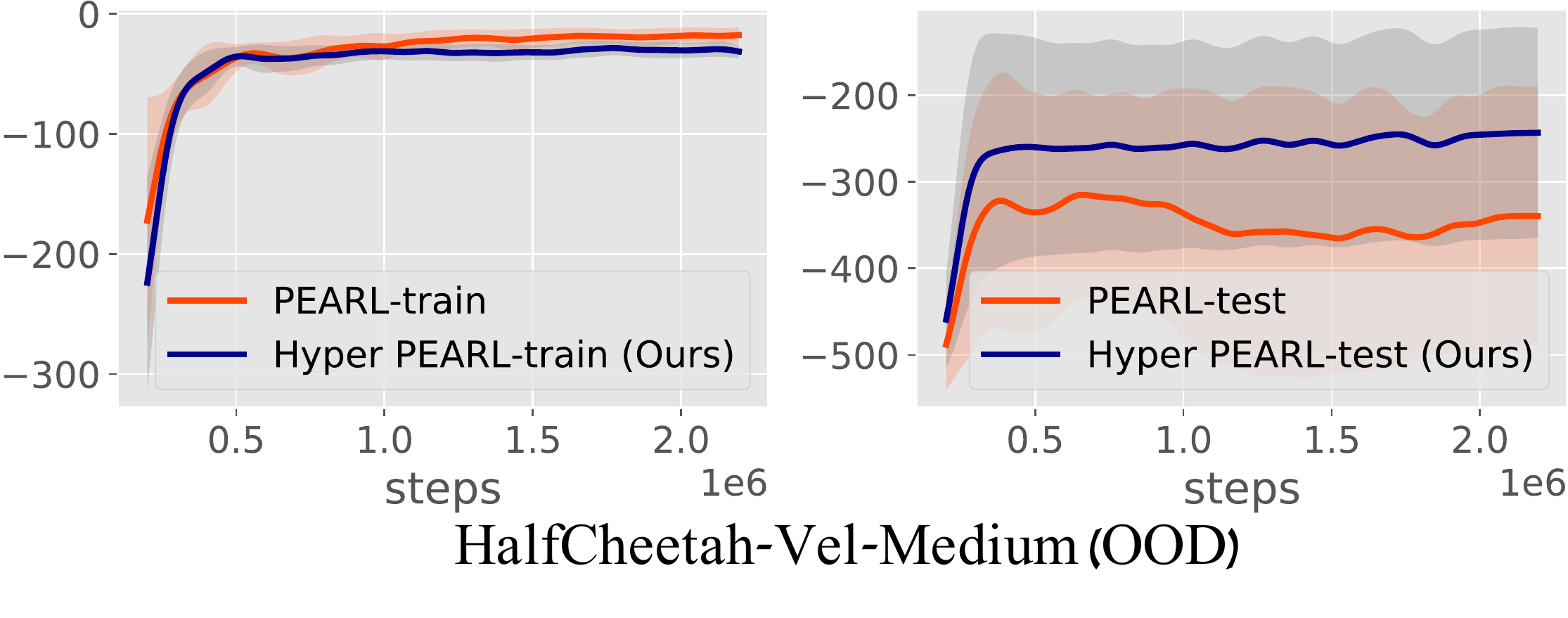}
\end{center}
\caption{PEARL results in an \textbf{Out Of Distribution} environment, HalfCheetah-Vel-Medium, where the training tasks' target is [0,2.5] and the test tasks' target is [2.5,3]. The Hypernetwork achieved slightly lower returns over the training tasks, yet it generalizes better over the OOD test tasks.}
\label{fig:OOD-PEARL}
\end{figure}

\section{Conclusions}
\label{sec:conclusions}

In this work, we set out to study neural models for the RL building blocks: $Q$-functions and meta-policies. Arguing that the unique nature of the RL setting requires unconventional models, we suggested the Hypernetwork model and showed empirically several significant advantages over MLP models. First, Hypernetworks are better able to estimate the parametric gradient signal of the $Q$-function required to train actor-critic algorithms. Second, they reduce the gradient variance in training meta-policies in Meta-RL. Finally, they improve OOD generalization and they do not require any adaptation step in Meta-RL training, which significantly facilitates the training process.

\section{Code}
Our Hypernetwork PyTorch implementation is found at \url{https://github.com/keynans/HypeRL}.

\newpage

\bibliography{mybib}
\bibliographystyle{apalike}

\onecolumn

\appendix

\section{Proof of Proposition 1}

\newtheorem{g_accuracy}{Proposition}

\begin{g_accuracy}
Let $\pi(a|s) = \mu_{\phi}(\varepsilon|s)$ be a stochastic parametric policy with $\varepsilon \sim p_{\varepsilon}$ and $\mu_{\phi}(\cdot|s)$ a transformation with a Lipschitz continuous gradient and a Lipschitz constant $\kappa_{\mu}$. Assume that its $Q$-function $Q^{\pi}(s,a)$ has a Lipschitz continuous gradient in $a$, i.e. $|\nabla_a Q^{\pi}(s,a_1) - \nabla_a Q^{\pi}(s,a_2)| \leq \kappa_q \|a_1 - a_2\|$. Define the average gradient operator $\overline{\nabla}_{\phi} f = \mathbb{E}_{s\sim\mathcal{D}}\left[\mathbb{E}_{\varepsilon\sim p_{\varepsilon}}[\nabla_{\phi} \mu_{\phi}(\varepsilon| s) \cdot f(s, \mu_{\phi}(\varepsilon|s))]\right]$. If there exists a gradient estimation $g(s,a)$ and $0 < \alpha < 1$ s.t.
\begin{equation}
\label{eq:alpha_rmse}
    \| \overline{\nabla}_{\phi}\cdot g - \overline{\nabla}_{\phi} \cdot \nabla_{a} Q^{\pi} \| \leq \alpha \| \overline{\nabla}_{\phi} \cdot \nabla_{a} Q^{\pi} \|
\end{equation}
then the ascent step $\phi' \leftarrow \phi + \eta \overline{\nabla}_{\phi}\cdot g$ with $\eta \leq \frac{1}{\tilde{k}}\frac{1-\alpha}{(1 + \alpha)^2}$ yields a positive empirical advantage policy.
\end{g_accuracy}

\begin{proof}

First, recall the objective to be optimized:
\begin{equation}
\begin{aligned}
    J(\phi) &= \mathbb{E}_{s\sim\mathcal{D}}\left[\mathbb{E}_{\varepsilon\sim p_{\varepsilon}}\left[ Q^{\pi}(s, \mu_{\phi}(\varepsilon; s)) \right]\right] \\
    \nabla_{\phi} J(\phi) &= \mathbb{E}_{s\sim\mathcal{D}}\left[\mathbb{E}_{\varepsilon\sim p_{\varepsilon}}\left[ \nabla_{\phi} \mu_{\phi}(\varepsilon; s) \cdot \nabla_a Q^{\pi}(s, \mu_{\phi}(\varepsilon; s)) \right]\right] = \overline{\nabla}_{\phi}\cdot \nabla_a Q^{\pi}
\end{aligned}
\end{equation}

Notice that as $Q^{\pi}$ is bounded by the maximal reward and its gradient is Lipschitz continuous, the gradient $\nabla_{a}Q^{\pi}$ is therefore bounded. Similarly, since the action space is bounded, and the terministic transformation $\mu_{\pi}$ has a Lipschitz continuous gradient, it follows that $\nabla_{\phi}\mu_{\pi}$ is also bounded. Define $\|\nabla_{a}Q^{\pi}\|\leq \sigma_q$ and $\|\nabla_{\phi}\mu_{\pi}\|\leq \sigma_{\mu}$.

\begin{lemma}
\label{lipschitz_sum}
Let $A(x): \mathbb{R}^n \to M^{k\times l}$ s.t. $\|A(x)\| \leq M_a$ and $\|A(x_1) - A(x_2)\| \leq \alpha \|x_1 - x_2\|$ and $\|\cdot\|$ is the induced vector norm.  And let $b(x): \mathbb{R}^n \to \mathbb{R}^l$ s.t. $\|b(x)\| \leq M_b$ and  $\|b(x_1) - b(x_2)\| \leq \beta \|x_1 - x_2\|$. The operator $c(x) = A(x) \cdot b(x):  \mathbb{R}^n \to \mathbb{R}^k$ is Lipschitz with constant $\kappa_c \leq \alpha M_a + \beta M_b$.
\end{lemma}
\begin{proof}
\begin{equation*}
        \begin{aligned}
         \|c(x_1) - c(x_2)\| &= \| A(x_1) \cdot b(x_1) - A(x_2) \cdot b(x_2) \| \\
        &= \| A(x_1) \cdot b(x_1) - A(x_1) \cdot b(x_2) + A(x_1) \cdot b(x_2) - A(x_2) \cdot b(x_2)  \| \\
        &\leq \| A(x_1) \cdot (b(x_1) -  b(x_2)) \| + \| (A(x_1) - A(x_2)) \cdot b(x_2) \| \\
        &\leq \| A(x_1) \| \|b(x_1) -  b(x_2)\| + \|A(x_1) - A(x_2)\| \|b(x_2)\| \\
        &\leq \left(\beta M_a + \alpha M_b \right)  \|x_1 - x_2\|
    \end{aligned}
\end{equation*}

\end{proof}

The Lipschitz constant of the objective gradient is bounded by
\begin{multline*}
    \|\nabla_{\phi} J(\phi_1) - \nabla_{\phi} J(\phi_2)\| = \left\| \mathbb{E} \left[ \nabla_{\phi} Q^{\pi}(s, \mu_{\phi_1}(\varepsilon; s)) - \nabla_{\phi} Q^{\pi}(s, \mu_{\phi_2}(\varepsilon; s))\right] \right\| \leq \\
     \mathbb{E} \left[ \left\| \nabla_{\phi} \mu_{\phi_1}(\varepsilon; s) \cdot \nabla_a Q^{\pi}(s, \mu_{\phi_1}(\varepsilon; s)) - \nabla_{\phi} \mu_{\phi_2}(\varepsilon; s) \cdot \nabla_a  Q^{\pi}(s, \mu_{\phi_2}(\varepsilon; s))\right\| \right] 
\end{multline*}
Applying Lemma \ref{lipschitz_sum}, we obtain
\begin{equation*}
    \|\nabla_{\phi} J(\phi_1) - \nabla_{\phi} J(\phi_2)\| \leq (\kappa_q \sigma_{\mu} + \kappa_{\mu} \sigma_q) \|\phi_1 - \phi_2\|.
\end{equation*}
Therefore, $J(\phi_1)$ is also Lipschitz. Hence, applying Taylor's expansion around $\phi$, we have that
\begin{equation*}
    J(\phi') \geq J(\phi) + (\phi' - \phi) \cdot \nabla_{\phi}J(\phi) - \frac{\kappa_J^2}{2} \|\phi' - \phi\|^2 \geq J(\phi) + (\phi' - \phi) \cdot \nabla_{\phi}J(\phi) - \frac{(\kappa_q \sigma_{\mu} + \kappa_{\mu} \sigma_q)^2}{2} \|\phi' - \phi\|^2.
\end{equation*}
Plugging in the iteration $\phi' \leftarrow \phi + \eta \overline{\nabla}_{\phi}\cdot g$ we obtain
\begin{equation}
\label{eq:one_step}
     J(\phi') \geq J(\phi) + \eta \left(\overline{\nabla}_{\phi}\cdot g\right) \cdot \left(\overline{\nabla}_{\phi}\cdot Q^{\pi}\right) - \frac{\eta^2 (\kappa_q \sigma_{\mu} + \kappa_{\mu} \sigma_q)^2}{2} \|\overline{\nabla}_{\phi}\cdot g\|^2. 
\end{equation}

Taking the second term on the right-hand side,
\begin{equation*}
\begin{aligned}
          \left(\overline{\nabla}_{\phi}\cdot g\right) \cdot \left(\overline{\nabla}_{\phi}\cdot Q^{\pi}\right) &=\left(\overline{\nabla}_{\phi}\cdot Q^{\pi} - (\overline{\nabla}_{\phi}\cdot g - \overline{\nabla}_{\phi}\cdot Q^{\pi})\right) \cdot \left(\overline{\nabla}_{\phi}\cdot Q^{\pi}\right) \\
     &\geq \left\|\overline{\nabla}_{\phi}\cdot Q^{\pi}\right\|^2 - \left\|\left(\overline{\nabla}_{\phi}\cdot g - \overline{\nabla}_{\phi}\cdot Q^{\pi}\right) \cdot \left(\overline{\nabla}_{\phi}\cdot Q^{\pi}\right)\right\| \\
     &\geq \left\|\overline{\nabla}_{\phi}\cdot Q^{\pi}\right\|^2 - \left\|\overline{\nabla}_{\phi}\cdot g - \overline{\nabla}_{\phi}\cdot Q^{\pi}\right\| \cdot \left\| \overline{\nabla}_{\phi}\cdot Q^{\pi}\right\| \\
     &\geq (1 - \alpha )\left\|\overline{\nabla}_{\phi}\cdot Q^{\pi}\right\|^2.
\end{aligned}
\end{equation*}
For the last term we have
\begin{equation*}
\begin{aligned}
          \left\|\overline{\nabla}_{\phi}\cdot g\right\|^2  &=\left\|\overline{\nabla}_{\phi}\cdot Q^{\pi} - (\overline{\nabla}_{\phi}\cdot g - \overline{\nabla}_{\phi}\cdot Q^{\pi})\right\|^2  \\
          & = \left\|\overline{\nabla}_{\phi}\cdot Q^{\pi}\right\|^2 - 2 \left(\overline{\nabla}_{\phi}\cdot Q^{\pi}\right) \cdot (\overline{\nabla}_{\phi}\cdot g - \overline{\nabla}_{\phi}\cdot Q^{\pi})  + \left\|\overline{\nabla}_{\phi}\cdot g - \overline{\nabla}_{\phi}\cdot Q^{\pi}\right\|^2 \\
          & \leq \left\|\overline{\nabla}_{\phi}\cdot Q^{\pi}\right\|^2 + 2 \left\|\overline{\nabla}_{\phi}\cdot Q^{\pi}\right\| \cdot \left\| \overline{\nabla}_{\phi}\cdot g - \overline{\nabla}_{\phi}\cdot Q^{\pi}\right\|  + \alpha^2 \left\|\overline{\nabla}_{\phi}\cdot Q^{\pi}\right\|^2 \\
          & \leq (1 + 2 \alpha + \alpha^2) \left\|\overline{\nabla}_{\phi}\cdot Q^{\pi}\right\|^2. 
\end{aligned}
\end{equation*}
Plugging both terms together into Eq. (\ref{eq:one_step}) we get
\begin{equation*}
    J(\phi') \geq J(\phi) + \left\|\overline{\nabla}_{\phi}\cdot Q^{\pi}\right\|^2 \left( \eta (1 - \alpha ) - \frac{1}{2} \eta^2(\kappa_q \sigma_{\mu} + \kappa_{\mu} \sigma_q)^2(1 + \alpha)^2  \right). 
\end{equation*}

To obtain a positive empirical advantage we need
\begin{equation*}
    \eta (1 - \alpha ) - \frac{1}{2} \eta^2(\kappa_q \sigma_{\mu} + \kappa_{\mu} \sigma_q)^2(1 + \alpha)^2 \geq 0 
\end{equation*}
Thus the sufficient requirement for the learning rate is
\begin{equation*}
    \eta \leq \frac{1}{\tilde{k}}\frac{1-\alpha}{(1 + \alpha)^2}.
\end{equation*}
where $\tilde{k} = \frac{1}{2} (\kappa_q \sigma_{\mu} + \kappa_{\mu} \sigma_q)$.

\end{proof}

\newpage

\section{Cosine Similarity Estimation}
\label{sec:gradient}

To evaluate the averaged Cosine Similarity (CS)
\begin{equation}
    cs(Q^{\pi}_{\theta}) = \mathbb{E}_{s\sim\mathcal{D}}\left[\frac{\nabla_a Q_{\theta}^{\pi}(s,a_{\mu}) \cdot \nabla_a Q^{\pi}(s,a_{\mu})}{\|\nabla_a Q_{\theta}^{\pi}(s,a_{\mu})\| \ \| \nabla_a Q^{\pi}(s,a_{\mu})\|}\right],
\end{equation}
we need to estimate the local CS for each state. To that end, we estimate the ``true" $Q$-function $Q^{\pi}(s,a)$ at the vicinity of $a=a_{\mu}$ with a non-parametric local linear model
\begin{equation*}
Q^{\pi}(s,a) \simeq f_{s,a_{\mu}}(a) = a \cdot g 
\end{equation*}
where $g\in\mathbb{R}^{N_a}$ s.t. the $Q$-function gradient is constant $\nabla_a Q^{\pi}(s,a) \simeq g$. To fit the linear model, we sample $N_r$ unbiased samples of the $Q$-function around $a_{\mu}$, i.e. $q_i = \hat{Q}^{\pi}(s,a_i)$. These samples are the empirical discounted sum of rewards following execution of action $a_i = a_{\mu} + \Delta_i$ at state $s$ and then applying policy $\pi$.

To fit the linear model we directly fit the constant model $g$ for the gradient. Recall that applying the Taylor's expansion around $a_{\mu}$ gives
\begin{equation*}
Q^{\pi}(s,a) = Q^{\pi}(s,a_{\mu}) + (a - a_{\mu}) \cdot  \nabla_a Q_{\theta}^{\pi}(s,a_{\mu}) + O\left(\|a - a_{\mu}\|^2\right)
\end{equation*},
therefore
\begin{equation*}
Q^{\pi}(s,a_2) - Q^{\pi}(s,a_1) - (a_2 - a_1) \cdot \nabla_a Q_{\theta}^{\pi}(s,a_{\mu})  =  O\left(\|a_2 - a_1\|^2\right)
\end{equation*}
for $a_1,a_2$ at the vicinity of $a_{\mu}$.

To find the best fit $g\simeq \nabla_a Q_{\theta}^{\pi}(s,a_{\mu})$ we minimize averaged the quadratic error term over all pairs of sampled trajectories
\begin{equation*}
g^{*}=\arg\min_g \sum_{i}^{N_r} \sum_{j}^{N_r} |(a_j - a_i) \cdot g - q_j + q_i |^2 .
\end{equation*}
This problem can be expressed in a matrix notation as
\begin{equation*}
g^{*} = \arg\min_g \left\| \tilde{X} g - \delta \right\|^2,
\end{equation*}
where $\tilde{X}\in{R}^{N_r^2 \times N_a}$ is a matrix with $N_r^2$ rows of all the vectors $a_j-a_i$ and $\delta$ is a $N_r^2$ element vector of all the differences $q_j-q_i$. Its minimization is the Least-Mean-Squared Estimator (LMSE)
\begin{equation*}
g^{*} = (\tilde{X}^T\tilde{X})^{-1}\tilde{X}^T\delta.
\end{equation*}
 
In our experiments we evaluated the CS every $K=10^4$ learning steps and used $N_s=15$, $N_r=15$ and $\Delta_i\sim\mathcal{N}(0, 0.3)$ for each evaluation. This choice trades off somewhat less accurate local estimators with more samples during training. To test our gradient estimator, we first applied it to the outputs of the $Q$-function network (instead of the true returns) and calculated the CS between a linear model based on the network outputs and the network parametric gradient. The results in Fig. \ref{fig:cs_q} show that our $g^*$ estimator obtains a high CS between the $Q$-net outputs of the SA-Hyper and MLP models and their respective parametric gradients. This indicates that these networks are locally ($\Delta \propto 0.3$) linear. On the other hand, the CS between the linear model based on the AS-Hyper outputs and its parametric gradient is lower, which indicates that the network is not necessarily close to linear with $\Delta \propto 0.3$. We assume that this may be because the action in the AS-Hyper configuration plays the meta-variable role which increases the non-linearity of the model with respect to the action input. Importantly, note that this does not indicate that the true $Q$-function of the AS-Hyper model is more non-linear than other models.
\begin{figure*}[hbt!]
\begin{center}
    \includegraphics[width=1.\linewidth]{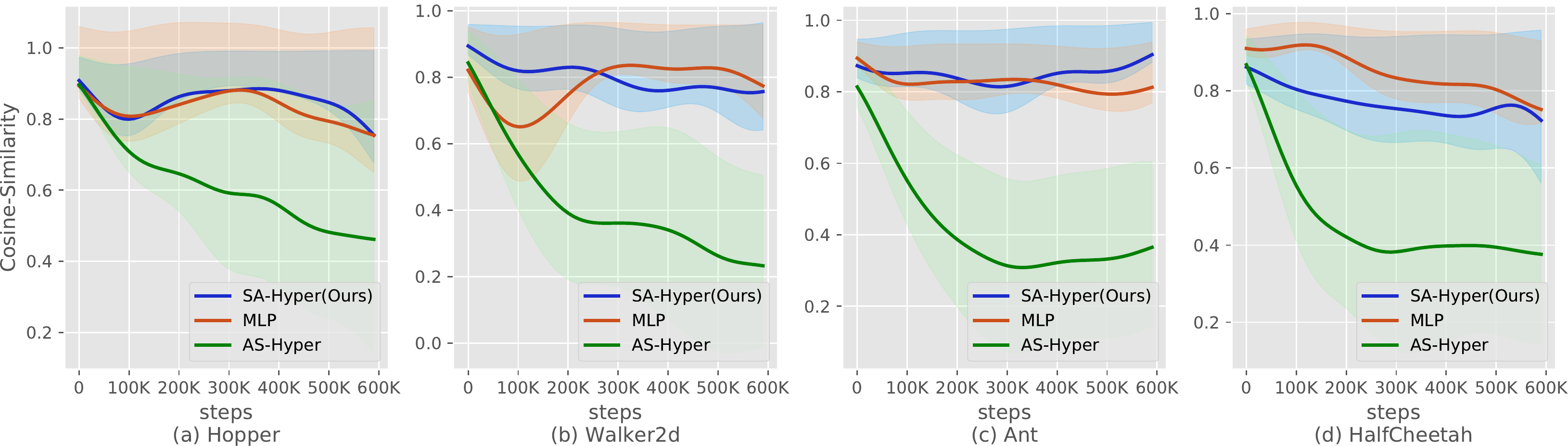}
\end{center}
\caption{The Cosine-Similarity between the LMSE estimator of the $Q$-net outputs and the parametric gradient averaged with a window size of $W=20$.}
\label{fig:cs_q}
\end{figure*}

In Fig. \ref{fig:cs_r} we plot the CS for 4 different environments averaged with a window size of $W=20$. The results show that on average the SA-Hyper configuration obtains a higher CS, which indicates that the policy optimization step is more accurate s.t. the RL training process is more efficient.
\begin{figure*}[hbt!]
\begin{center}
    \includegraphics[width=1.\linewidth]{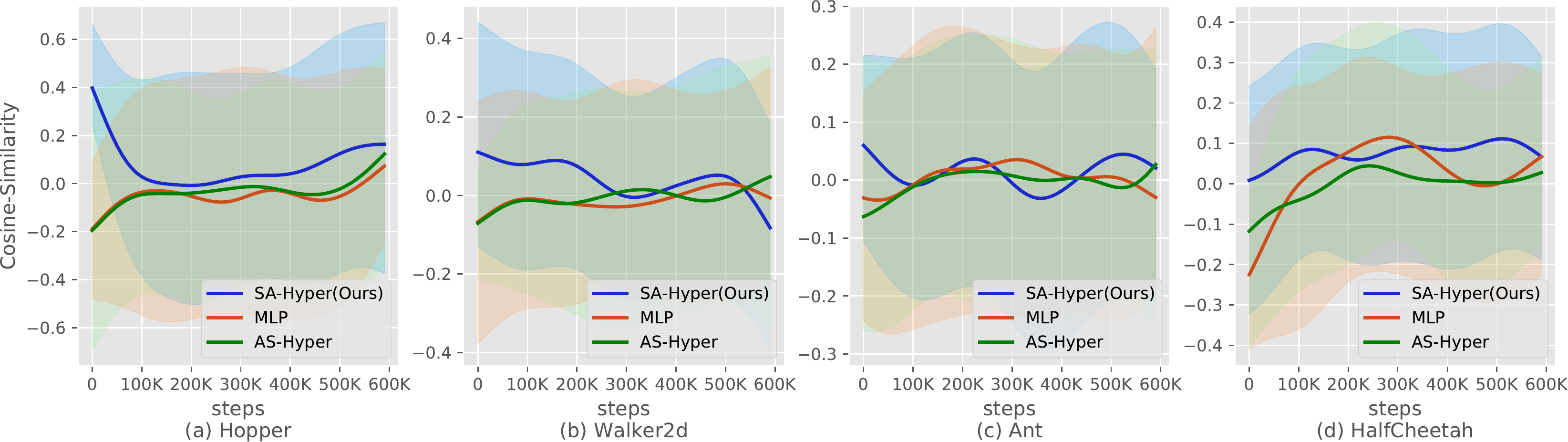}
\end{center}
\caption{The Cosine-Similarity between the LMSE estimator of the empirical sum of rewards and the parametric gradient averaged with a window size of $W=20$.}
\label{fig:cs_r}
\end{figure*}

\newpage
\ 
\newpage

\section{Gradient Step Noise Statistics in MAML}
Hypernetworks disentangle the state-dependent gradient and the task-dependent gradient. As explained in the paper, we hypothesized that these characteristics reduce the gradient ascent step noise during policy updates
\begin{equation*}
    \phi \leftarrow \phi - \eta \widehat{\nabla}_{\phi} J
\end{equation*}
where $\widehat{\nabla}_{\phi} J$ is the gradient step estimation and $\eta$ is the learning rate. It is not obvious how to define the gradient noise properly as any norm-based measure depends on the network's structure and size. Therefore, we take an alternative approach and define the gradient noise as the performance statistics after applying a set of independent gradient steps. In simple words, this definition essentially corresponds to how noisy the learning process is.

To estimate the performance statistics, we take $N=50$ different independent policy gradients based on independent trajectories at 4 different time steps during the training process. For each gradient step, we sampled 20 trajectories with a maximal length of 200 steps (identical to a single policy update during the training process) out of 40 tasks. After each gradient step, we evaluated the performance and restored the policy's weights s.t. the gradient steps are independent.

We compared two different network architectures, both with access to an oracle context: (1) Hyper-MAML; and (2) Context-MAML. We did not evaluate Vanilla-MAML as it has no context and the gradient noise, in this case, might also be due to higher adaptation noise as the context must be recovered from the trajectories' rewards. In the paper, we presented the performance statistics after $N=50$ different updates. In Table \ref{MAML_grad_table} we present the variance of those statistics.

\begin{table}[hbt!]
\caption{The gradient coefficient of variation  $\frac{\sigma}{|\mu|}$ and the variance (in brackets) in MAML. Hyper-MAML refers to Hypernetwork policy where the oracle-context is the meta-variable and the state features are the base-variable. Context-MAML refers to the MLP model policy where the oracle-context is concatenated with the state features. To compare between different policies with different reward scales, we report both the coefficient of variation and the variance in brackets.}
\label{MAML_grad_table}
\begin{center}
\begin{tabular}{lllll}
\multicolumn{1}{c}{\bf Envrionment}  &\multicolumn{1}{c}  {\bf 50 iter}  &\multicolumn{1}{c}{\bf 150 iter}    &\multicolumn{1}{c}{\bf 300 iter}    &\multicolumn{1}{c}{\bf 450 iter}
\\ \hline \\
HalfCheetah-Fwd-Back & & & &\\
\hline
Context-MAML &1.184 (774)           &4.492 (2595)             &2.590 (1891)                  &0.822 (3689)  \\
Hyper MAML (Ours)   &\textbf{0.027} (26)  &\textbf{0.017} (43)    &\textbf{0.021} (96)          &\textbf{0.014} (53) \\
\hline \\
HalfCheetah-Vel & & & &\\
\hline
Context-MAML &0.035 (122)           &0.050 (208)           &0.093 (520)                 &0.066 (161) \\
Hyper MAML (Ours)   &\textbf{0.009} (5)   &\textbf{0.005} (1)    &\textbf{0.008} (2)           &\textbf{0.009} (2) \\
\hline \\
Ant-Fwd-Back & & & &\\
\hline
Context-MAML &0.274 (3)           &0.199 (5)           &0.400 (12)                   &0.285 (20)\\
Hyper MAML (Ours)    &\textbf{0.073} (1)   &\textbf{0.047} (2)    &\textbf{0.050} (6)          &\textbf{0.047} (11) \\
\hline \\
Ant-Vel & & & &\\
\hline
Context-MAML &0.379 (52)          &0.377 (8)             &0.628 (109)                   &0.418 (117) \\
Hyper MAML (Ours)    &\textbf{0.252} (5)   &\textbf{0.159} (2)    &\textbf{0.080} (2)           &\textbf{0.057} (2) \\

\end{tabular}
\end{center}
\end{table}

\newpage

\section {Models Design}

\subsection{Hypernetwork Architecture}
\label{sec:compositions}

The Hypernetwork's primary part is composed of three main blocks followed by a set of heads. Each block contains an up-scaling linear layer followed by two pre-activation residual linear blocks (ReLU-linear-ReLU-linear). The first block up-scales from the state's dimension to 256 and the second and third blocks grow to 512 and 1024 neurons respectively. The total number of learnable parameters in the three blocks is $\sim6.5M$. The last block is followed by the heads which are a set of linear transformations that generate the $\sim2K$ dynamic parameters (including weights, biases and gains). The heads have $\sim2.5M$ learnable parameters s.t. the total number of parameters in the primary part is $\sim9M$.

\subsection{Primary Model Design: Negative Results}

In our search for a primary network that can learn to model the weights of a state-dependent dynamic $Q$-function, we experimented with several different architectures. Here we outline a list of negative results, i.e. models that failed to learn good primary networks.
\begin{enumerate}
    \item we tried three network architecture: (1) MLP; (2) Dense Blocks \cite{huang2017densely}; and (3) ResNet Blocks \cite{he2016deep}. The MLP did not converge and the dense blocks were sensitive to the initialization with spikes in the policy's gradient which led to an unstable learning process.
    \item We found that the head size (the last layer that outputs all the dynamic network weights) should not be smaller than 512 and the depth should be at least 5 blocks. Upsampling from the low state dimension can either be done gradually or at the first layer.
    \item We tried different normalization schemes: (1) weight normalization \cite{salimans2016weight}; (2) spectral normalization \cite{miyato2018spectral}; and (3) batch normalization \cite{ioffe2015batch}. All of them did not help and slowed or stopped the learning.
    \item For the non-linear activation functions, we tried RELU and ELU which we found to have similar performances.
\end{enumerate}

\subsection{Hypernetwork Initialization}

A proper initialization for the Hypernetwork is crucial for the network's numerical stability and its ability to learn. Common initialization methods are not necessarily suited for Hypernetworks \cite{chang2019principled} since they fail to generate the dynamic weights in the correct scale. We found that some RL algorithms are more affected than others by the initialization scheme, e.g, SAC is more sensitive than TD3. However, we leave this question of why some RL algorithms are more sensitive than others to the weight initialization for future research. 

To improve the Hypernetwork weight initialization, we followed \cite{hyper_bayesian} and initialized the primary weights with smaller than usual values s.t. the initial dynamic weights were also relatively small compared to standard initialization (Fig. \ref{fig:dynamic_init}). As is shown in Fig. \ref{fig:dynamic_init_100k}, this enables the dynamic weights to converge during the training process to a relatively similar distribution of a normal MLP network. 

The residual blocks in the primary part were initialized with a fan-in Kaiming uniform initialization \cite{he2015delving} with a gain of $\frac{1}{\sqrt{12}}$ (instead of the normal gain of $\sqrt{2}$ for the ReLU activation). We used fixed uniform distributions to initialize the weights in the heads: $U(-0.05, 0.05)$ for the first dynamic layer, $U(-0.008, 0.008)$ for the second dynamic layer and for the standard deviation output layer in the PEARL meta-policy we used the $U(-0.001, 0.001)$ distribution.

In Fig. \ref{fig:dynamic_init} and  Fig. \ref{fig:dynamic_init_100k} we plot the histogram of the TD3 critic dynamic network weights with different primary initializations: (1) our custom primary initialization; and (2) The default Pytorch initialization of the primary network. We compare the dynamic weights to the weights of a standard MLP-Small network (the same size as the dynamic network). We take two snapshots of the weight distribution: (1) in Fig. \ref{fig:dynamic_init} before the start of the training process; and (2) after $100K$ training steps. In Table \ref{init_table} we also report the total-variation distance between each initialization and the MLP-Small weight distribution. Interestingly, the results show that while the dynamic weight distribution with the Pytorch primary initialization is closer to the MLP-Small distribution at the beginning of the training process, after 100K training steps our primary initialized weights produce closer dynamic weight distribution to the MLP-Small network (also trained for $100K$ steps).

\begin{figure*}[h!]
\begin{center}
    \includegraphics[width=\linewidth]{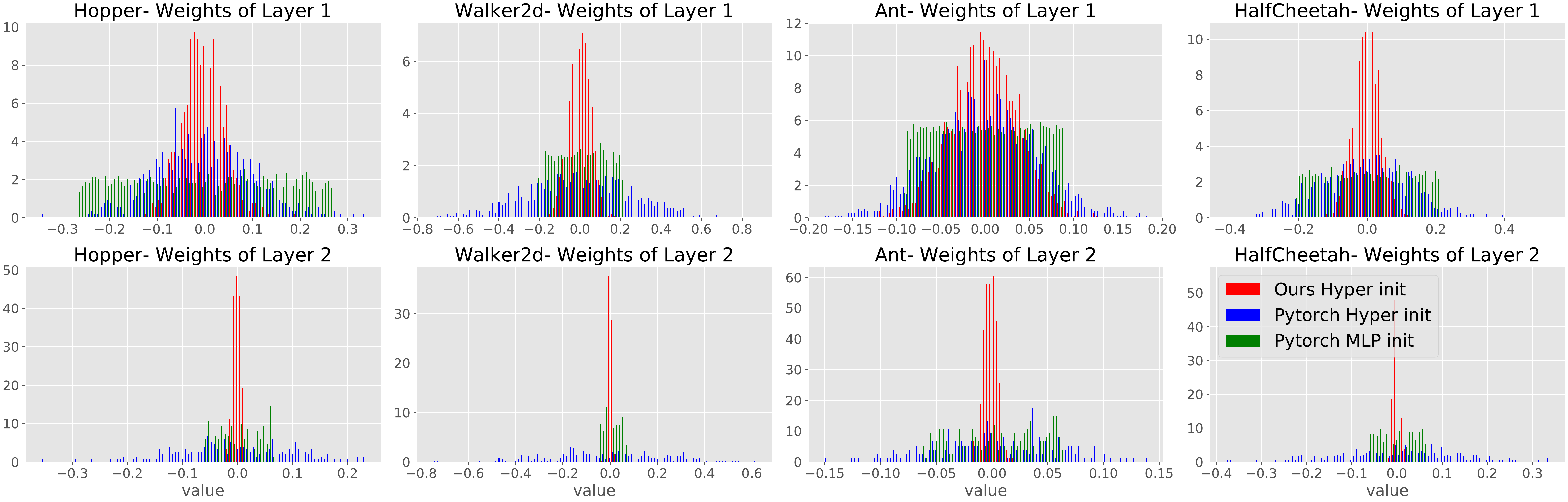}
\end{center}
\caption{Dynamic network weight distribution of different initialization schemes \textbf{at the beginning of the training}. The ``Hyper init" refers to a primary network initialized with our suggested initialization scheme. 'Pytorch Hyper init' refers to the Pytorch default initialization of the primary network and ``Pytorch MLP init" refers to the Pytorch default initialization of the MLP-Small model (same architecture as the dynamic network).}
\label{fig:dynamic_init}
\end{figure*}

\begin{figure*}[h!]
\begin{center}
    \includegraphics[width=\linewidth]{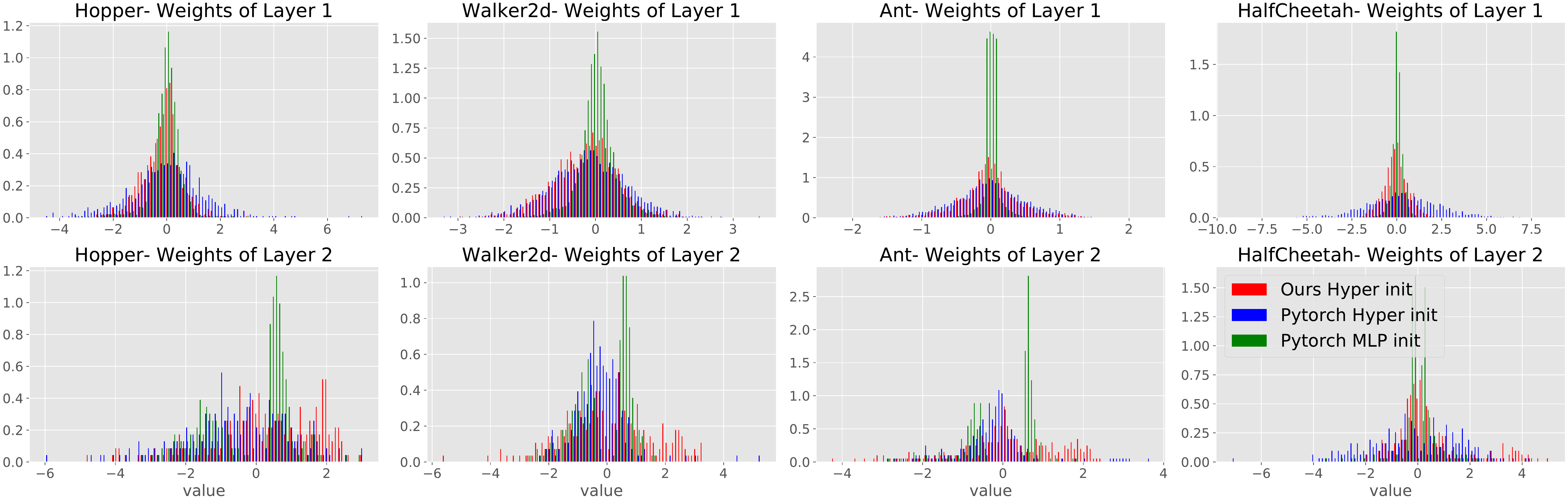}
\end{center}
\caption{Dynamic network weight distribution of different initialization schemes \textbf{after $\textbf{100K}$ training steps}. The ``Hyper init" refers to a primary network initialized with our suggested initialization scheme. 'Pytorch Hyper init' refers to the Pytorch default initialization of the primary network and 'Pytorch MLP init' weight distribution of the MLP-Small model (same architecture as the dynamic network).}
\label{fig:dynamic_init_100k}
\end{figure*}

\begin{table}[hbt!]
\caption{Total-variation distance between the dynamic weight distribution and the ``Pytorch MLP init" weight distribution: at the beginning of the training process we find that 'Pytorch Hyper init' is closer to the 'Pytorch MLP init' weight distribution while after $100K$ training steps we find that our initialization is closer to the 'Pytorch MLP init' weights (also trained for $100K$ steps).}
\label{init_table}
\begin{center}
\begin{tabular}{lcccc}
\multicolumn{1}{c}{\bf Primary Initialization Scheme}  &\multicolumn{1}{c}  {\bf Hopper}  &\multicolumn{1}{c}{\bf Walker2d}    &\multicolumn{1}{c}{\bf Ant}    &\multicolumn{1}{c}{\bf HalfCheetah}
\\ \hline \\
\textbf{First Layer} \\
\hline
Ours Hyper init  &31.4 &23.9 &13.6 &29.4\\
Pytorch Hyprer init &\textbf{16.3} &\textbf{20.5} &\textbf{9.2} &\textbf{8.8} \\
\hline \\
\textbf{Second Layer} \\
\hline
Ours Hyper init  &34.8 &\textbf{30.77} &37.7 &36.9\\
Pytorch Hyprer init &\textbf{24.7} &39.6 &\textbf{11.2} &\textbf{29.3} \\
\hline \\
\textbf{First Layer After 100K Steps} \\
\hline
Ours Hyper init &\textbf{14.4} &\textbf{19.6} &\textbf{29.9} &\textbf{16.4}\\
Pytorch Hyprer init &24.9 &22.6 &34.0 &22.4 \\
\hline \\
\textbf{Second Layer After 100K Steps} \\
\hline
Ours Hyper init &\textbf{31.2} &28.5 &\textbf{30.6} &\textbf{21.1}\\
Pytorch Hyprer init &32.11 &\textbf{20.8} &30.7 &31.1 \\

\end{tabular}
\end{center}
\end{table}

\subsection{Baseline Models for the SAC and TD3 algorithms}
In our TD3 and SAC experiments, we tested the Hypernetwork architecture with respect to 7 different baseline models.

\subsubsection{MLP-Standard}
A standard MLP architecture, which is used in many RL papers (e.g. SAC and TD3) with 2 hidden layers of 256 neurons each with ReLU activation function. 

\subsubsection{MLP-Small}
The MLP-Small model helps in understanding the gain of using context-dependent dynamic weights. It is an MLP network with the same architecture as our dynamic network model, i.e. 1 hidden layer with 256 neurons followed by a ReLU activation function. As expected, although the MLP-Small and MLP-Standard configurations are relatively similar with only a different number of hidden layers (1 and 2 respectively), the MLP-Small achieved close to half the return of the MLP-Standard. However, our experiments show that when using even a shallow MLP network with context-dependent weights (i.e. our SA-Hyper model), it can significantly outperform both shallow and deeper standard MLP models. 

\subsubsection{MLP-Large}
To make sure that the performance gain is not due to the large number of weights in the primary network, we evaluated MLP-Large, an MLP network with 2 hidden layers as the MLP-Standard but with 2,900 neurons in each layer. This yields a total number of $\sim9M$ learnable parameters, as in our entire primary model. While this large network usually outperformed other baselines, in almost all environments it still did not reach the Hypernetwork performance with one exception in the Ant-v2 environment in the TD3 algorithm. This provides another empirical argument that Hypernetworks are more suited for the RL problem and their performance gain is not only due to their larger parametric space. 

\subsubsection{ResNet Features}

To test whether the performance gain is due to the expressiveness of the ResNet model, we evaluated ResNet-Features: an MLP-Small model but instead of plugging in the raw state features, we use the primary model configuration (with ResNet blocks) to generate 10 learnable features of the state. Note that the feature extractor part of ResNet-Features has a similar parameter space as the Hypernetwork's primary model except for the head units. The ResNet-Features was unable to learn on most environments in both algorithms, even though we tried several different initialization schemes. This shows that the primary model is not suitable for a state's features extraction, and while it may be possible to find other models with ResNet that outperform this ResNet model, it is yet further evidence that the success of the Hypernetwork architecture is not attributed solely to the ResNet expressiveness power in the primary network.

\subsubsection{AS-Hyper}

This is the reverse configuration of our SA-Hyper model. In this configuration, the action is the meta-variable and the state serves as the base-variable. Its lower performance provides another empirical argument (alongside the lower CS, see Sec. \ref{sec:gradient}) that the ``correct" Hypernetwork composition is when the state plays the context role and the action is the base-variable.

\subsubsection{Emb-Hyper}

In this configuration, we replace the input of the primary network with a learnable embedding of size 5 (equal to the PEARL context size) and the dynamic part gets both the state and the action as its input variables. This produces a learnable set of weights that is constant for all states and actions. However, unlike MLP-Small, the weights are generated via the primary model and are not independent as in normal neural network training. Note that we did not include this experiment in the main paper but we have added it to the results in the appendix. This is another configuration that aims to validate that the Hypernetwork gain is not due to the over-parameterization of the primary model and that the disentanglement of the state and action is an important ingredient of the Hypernetwork performance.

\subsubsection{ResNet 35}
To validate that the performance gain is not due to a large number of weights in the primary network combined with the expressiveness of the residual blocks, we evaluated a full ResNet architecture: The state and actions are concatenated and followed by 35 ResNet blocks. Each block contains two linear layers of 256 size (and an identity path). This yields a a total number of $\sim4.5M$ learnable parameters, which is half of the $9M$ parameters in the Hypernetwork model. In almost all environments it underperformed both with respect to SA-Hyper and also with respect to the MLP-Standard baseline.

\subsubsection{Q-D2RL}
The Deep Dense architecture (D2RL) \cite{sinha2020d2rl} suggests to add skip connections from the input to each hidden layer. In the original paper this applies both to the $Q$-net model, where states and actions are concatenated and added to each hidden layer, and to policies where only states are added to each hidden layer. According to the paper, the best performing model contains 4 hidden layers. Here, we compared to Q-D2RL which only modifies the $Q$-net as our SA-Hyper model but does not alter the policy network. Q-D2RL shows an inconsistent performance between SAC and TD3. In the SAC algorithm, it performs close to the SA-Hyper in all environments. On the other hand, in the TD3 algorithm, Q-D2RL was unable to reach the SA-Hyper performance in any environment.

\subsection{Complexity and Run Time Considerations}

Modern deep learning packages such as Pytorch and Tensorflow currently do not have optimized implementation of Hypernetworks as opposed to conventional neural architectures such as CNN or MLP. Therefore, it is not surprising that the training of Hypernetwork can take a longer time than MLP models. However, remarkably, in MAML we were able to reduce the training time as the primary weights and gradients are calculated only once for each task and the dynamic network is smaller than the Vanilla-MAML MLP network. Therefore, within each task, both data collection and gradient calculation with the dynamic model requires less time than the Vanilla-MAML network. In Table \ref{run_time_table} we summarize the average training time of each algorithm and compare the Hyper and MLP configurations.

\begin{table}[hbt!]
\label{run_time_table}
\caption{Comparing the algorithms' average running time between Hyper and MLP models: Single iteration training time for the MAML algorithm and 5K steps training time for all other algorithms. Note that each agent was trained using a single NVIDIA® GeForce® RTX 2080 Ti GPU with a 11019 MiB memory.}
\label{run_time}
\begin{center}
\begin{tabular}{lcc}
\multicolumn{1}{c}{\bf Algorithm}  &\multicolumn{1}{c}  {\bf MLP}  &\multicolumn{1}{c}{\bf Hyper}
\\ \hline \\
SAC  &$120s$ &$200s$\\
TD3  &$40s$  &$140s$\\
PEARL             &$450s$ &$700s$\\
MAML            &$150s$ &$145s$\\
Multi-Task MAML &- &$120s$\\
\end{tabular}
\end{center}
\end{table}

\newpage

\section{Experiments}

In this section, we report the training results of all tested algorithms as well as the hyperparameters used in these experiments. For each algorithm, we plot the mean reward and standard deviation over five different seeds. The evaluation procedure of single task RL algorithms was done every $5K$ training steps, with a mean calculated over ten independent trajectory roll-outs, without exploration, as described in \cite{fujimoto2018addressing}. The evaluation procedure of the Meta-RL algorithms was done after every algorithm's iteration, with a mean calculated over all test tasks' roll-outs, as was done in \cite{rakelly2019efficient}. In 'Velocity' tasks in Meta-RL, we sample training and test tasks from $[0, 3]$ except for the HalfCheetah-Vel-Medium(OOD) environment which the training tasks sample from $[0, 2.5]$ and the test tasks sample from $[2.5, 3]$. We used $100$ training tasks and $30$ tests tasks for both algorithms (PEARL and MAML) on ``Velocity" tasks and $2$ tasks for the ``Direction" tasks, forward and backward.

\subsection{TD3}

\begin{figure*}[h!]
\begin{center}
    \includegraphics[width=\linewidth]{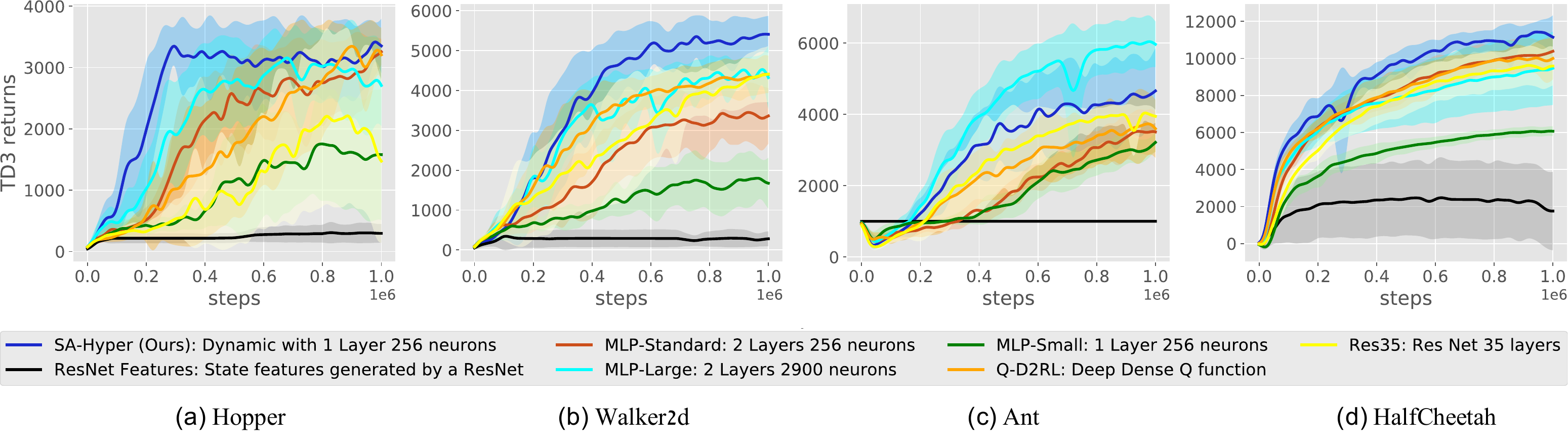}
\end{center}
\caption{TD3 performance of different MLP architectures compared to the SA-Hyper. SA-Hyper shows consistent high performance in all environments and outperforms all other architectures except for the Ant environment.}
\label{fig:TD3_architecture}
\end{figure*}

\begin{figure*}[hbt!]
\begin{center}
    \includegraphics[width=1.\linewidth]{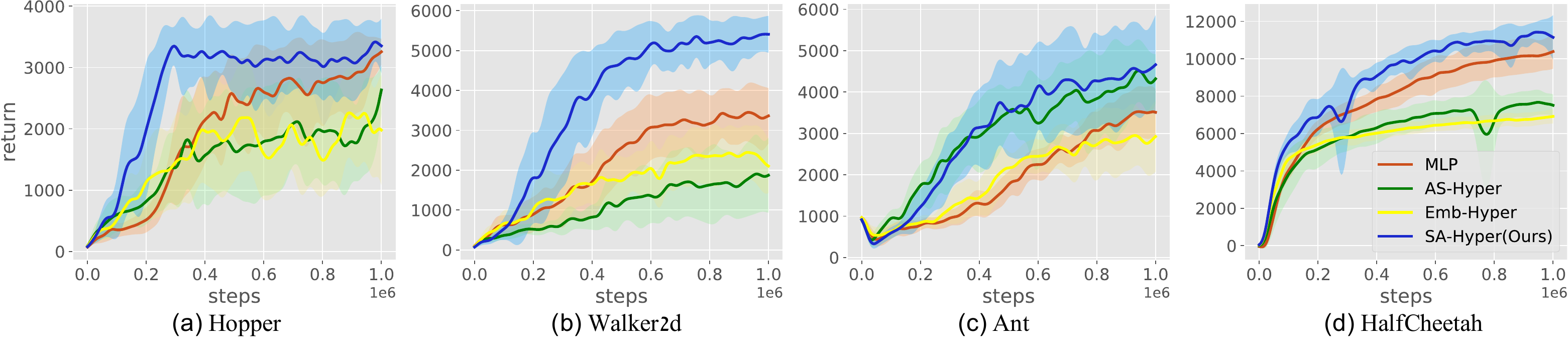}
\end{center}
\caption{TD3 Performance with Hypernetwork critic compared to MLP critic over different 'Mujoco' environments. In all environments, Hypernetwork outperforms all the baselines.}
\label{fig:TD3-emb}
\end{figure*}

\begin{table}[hbt!]
\caption{TD3 highest rewards.}
\label{tbl:TD3_highest_rewards}
\begin{center}
\begin{tabular}{lcccc}
\multicolumn{1}{c}{\bf Network}  &\multicolumn{1}{c}  {\bf Hopper}  &\multicolumn{1}{c}{\bf Walker2D}    &\multicolumn{1}{c}{\bf Ant}    &\multicolumn{1}{c}{\bf HalfCheetah}
\\ \hline \\
MLP-Standard  &$3256\pm211$       &$3449\pm730$             &$3524\pm617$           &$10384\pm923$ \\
MLP-Large   &$3156\pm368$       &$4527\pm397$             &$\textbf{6042}\pm\textbf{731}$           &$9467\pm1978$ \\
MLP-Small   &$1756\pm926$       &$1799\pm538m$             &$3215\pm267$           &$6071\pm256$ \\
ResNet-Features &$307\pm173$       &$343\pm349$             &$1001\pm1$           &$2474\pm2184$ \\
ResNet35 &$2213\pm1431$       &$4411\pm703$             &$4042\pm215$           &$9621\pm1072 $ \\
Q-D2RL &$3347\pm270$       &$4408\pm473$             &$3736\pm881$           &$10023\pm867$ \\
AS-Hyper    &$2633\pm391$       &$1905\pm985$             &$4513\pm759$           &$7669\pm667$ \\
Emb-Hyper    &$2261\pm728$       &$2446\pm676$             &$2949\pm741$           &$6915\pm374$ \\
SA-Hyper (Ours) &$\textbf{3418}\pm\textbf{318}$       &$\textbf{5412}\pm\textbf{445}$             &$4660\pm1194$           &$\textbf{11423}\pm\textbf{560}$ \\
\end{tabular}
\end{center}
\end{table}

\begin{table}[hbt!]
\caption{TD3 Hyper Parameters}
\label{TD3_hyper_parameters}
\begin{center}
\begin{tabular}{lcc}
\multicolumn{1}{c}{\bf Hyper-parameter} &\multicolumn{1}{c}{\bf TD3}    &\multicolumn{1}{c}{\bf Hyper TD3 (Ours)}
\\ \hline \\
Actor Learning Rate &$3e^{-4}$  &$3e^{-4}$ \\
Critic Learning Rate &$3e^{-4}$  &$5e^{-5}$  \\
Optimizer              &Adam  &Adam \\
Batch Size             &100  &100 \\
Policy update frequency      &2 &2 \\
Discount Factor            &0.99 &0.99 \\
Target critic update     &0.005 &0.005\\
Target policy update     &0.005 &0.005\\
Reward Scaling           &1  &1 \\
Exploration Policy        &$N(0,0.1)$  &$N(0,0.1)$ \\
\end{tabular}
\end{center}
\end{table}

\newpage
\
\newpage
\subsection{SAC}

\begin{figure*}[hbt!]
\begin{center}
    \includegraphics[width=\linewidth]{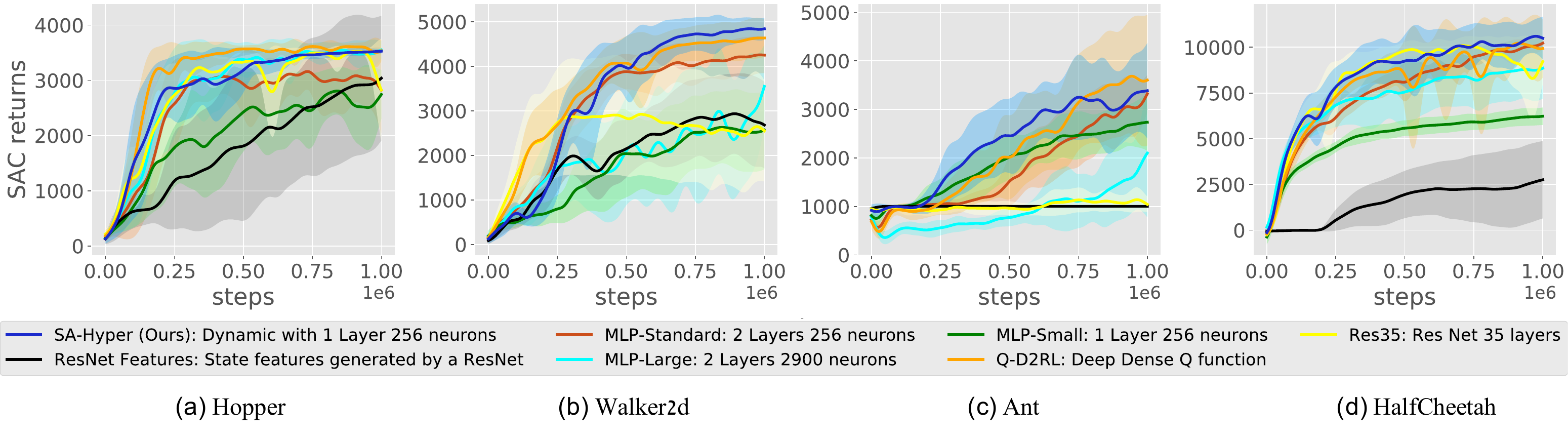}
\end{center}
\caption{SAC Performance of different critic models.}
\label{fig:fig:SAC_architecture}
\end{figure*}
\begin{table}[hbt!]
\caption{SAC highest rewards.}
\label{tbl:SAC_highest_reward}
\begin{center}
\begin{tabular}{lcccc}
\multicolumn{1}{c}{\bf Network}  &\multicolumn{1}{c}  {\bf Hopper}  &\multicolumn{1}{c}{\bf Walker2D}    &\multicolumn{1}{c}{\bf Ant}    &\multicolumn{1}{c}{\bf HalfCheetah}
\\ \hline \\
MLP-Standard  &$3160\pm327$       &$4258\pm413$             &$3323\pm389$           &$10225\pm324$ \\
MLP-Large   &$3549\pm160$       &$3550\pm936$             &$2100\pm1322$           &$8853\pm1663$ \\
MLP-Small   &$2806\pm425$       &$2629\pm804$             &$2735\pm589$           &$6229\pm475$ \\
ResNet-Features &$3038\pm1129$       &$2936\pm896$             &$1002\pm1$           &$2755\pm2114$ \\
ResNet35 &$3525\pm40$       &$2923\pm1369$             &$1138\pm252$           &$10096\pm468$ \\
Q-D2RL &$\textbf{3612}\pm\textbf{51}$       &$4638\pm441$             &$\textbf{3684}\pm\textbf{1207}$           &$10224\pm1090$ \\
SA-Hyper (Ours) &$3527\pm40$       &$\textbf{4844}\pm\textbf{254}$             &$3385\pm983$           &$\textbf{10600}\pm\textbf{950}$ \\
\end{tabular}
\end{center}
\end{table}

\begin{table}[hbt!]
\caption{SAC Hyper Parameters}
\label{single_hyper_parameters}
\begin{center}
\begin{tabular}{lcc}
\multicolumn{1}{c}{\bf Hyper-parameter}  &\multicolumn{1}{c}  {\bf SAC}  &\multicolumn{1}{c}{\bf Hyper SAC (Ours)}
\\ \hline \\
Actor Learning Rate &$3e^{-4}$ &$2e^{-5}$, $1e^{-4}$ for 'HalfCheetah' \ \\
Critic Learning Rate&$3e^{-4}$  &$5e^{-5}$ \\
Optimizer             &Adam &Adam \\
Batch Size            &256 &256  \\
Discount Factor            &0.99 &0.99 \\
Target critic update     &0.005 & 0.005\\
Reward Scaling           &5 &5\\
\end{tabular}
\end{center}
\end{table}

\newpage
\subsection{MAML}

\begin{figure*}[hbt!]
\begin{center}
    \includegraphics[width=1.\linewidth]{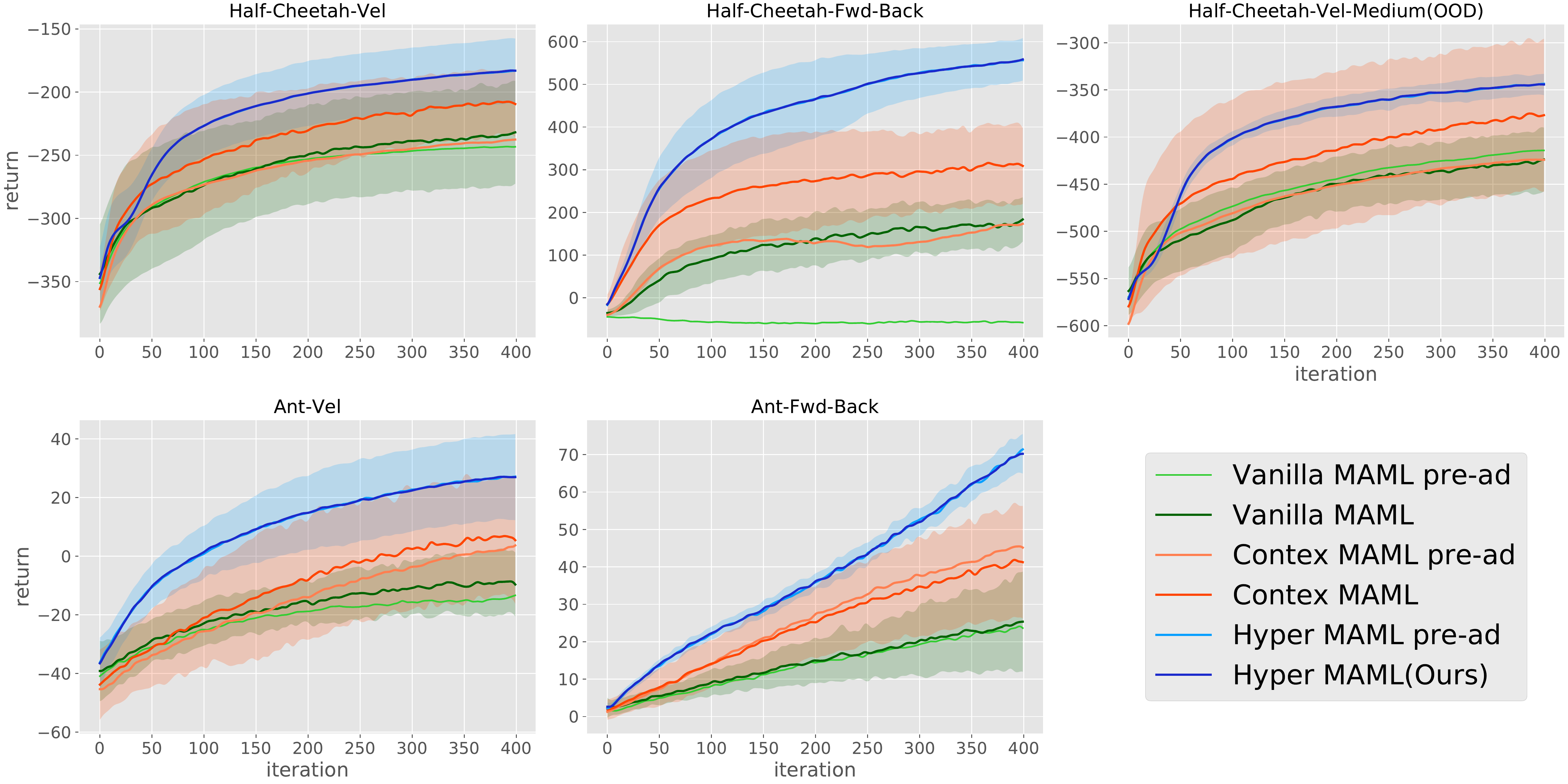}
\end{center}
\caption{MAML Performance over \textbf{test tasks} with a Hypernetwork policy compared to MLP policy with and without a given context of the tasks by an \textit{oracle}. The oracle-context improves the MAML performance but Hyper-MAML outperforms Context-MAML and, importantly, it does not require an adaptation step.}
\label{fig:MAML_test}
\end{figure*}

\begin{figure*}[hbt!]
\begin{center}
    \includegraphics[width=1.\linewidth]{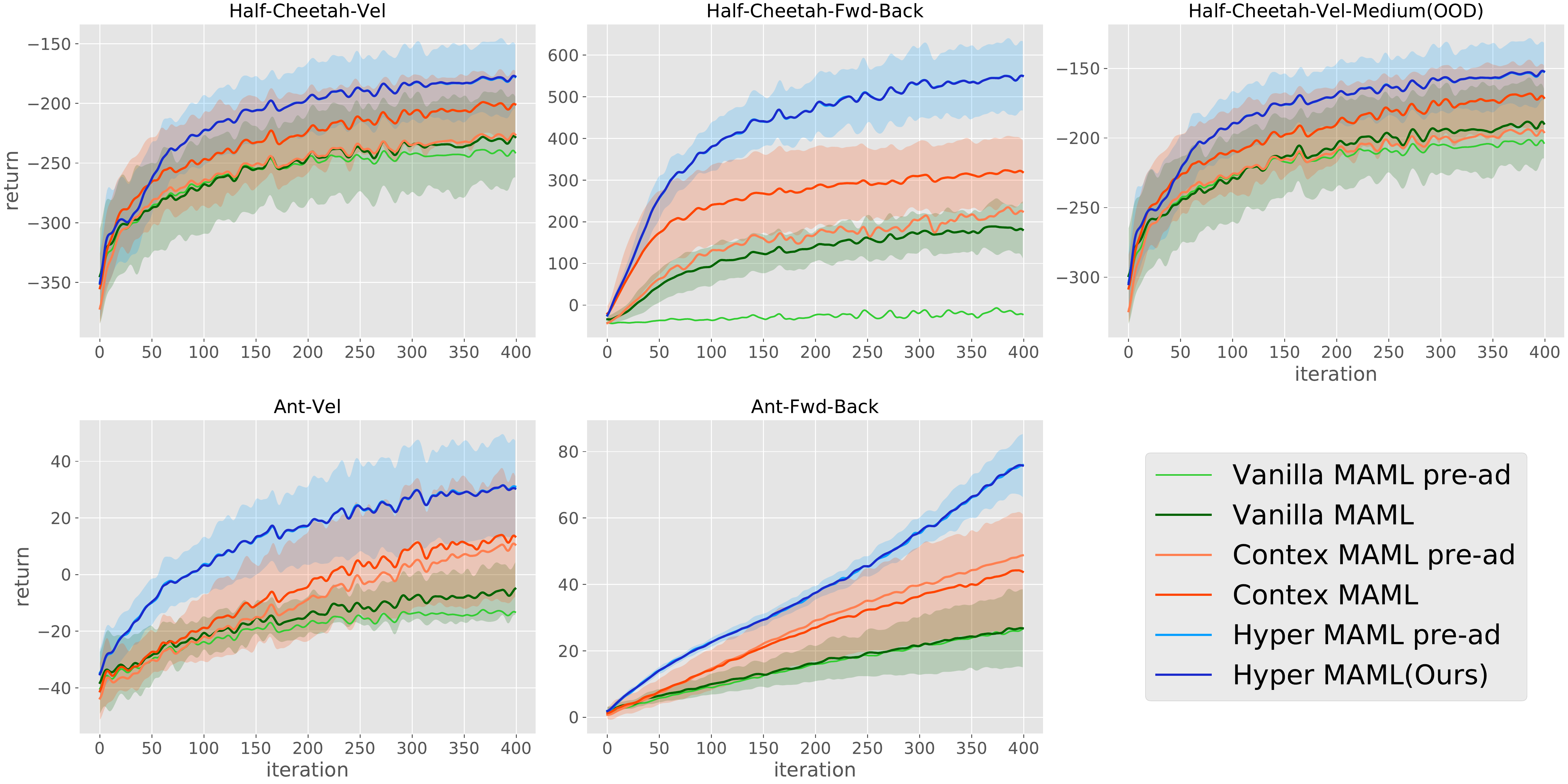}
\end{center}
\caption{MAML Performance over \textbf{training tasks} with a Hypernetwork policy compared to MLP policy with and without a given context of the tasks by an \textit{oracle}. The oracle-context improves the MAML performance but Hyper-MAML outperforms Context-MAML and, importantly, it does not require an adaptation step.}
\label{fig:MAML}
\end{figure*}

\subsubsection{Eliminating the adaptation step}

Our experiments show that taking the MAML adaptation step is unnecessary when using the Hyper-MAML model with an oracle-context (as opposed to Context-MAML which uses an oracle-context but still benefits from the adaptation step). We further investigate whether we can also eliminate the adaptation step during training s.t. the gradient of each task is calculated with the policy current weights as opposed to MAML which calculates the gradient at the policy's adapted weights. We term this method as Multi-Task Hyper-MAML (following \cite{fakoor2019meta} which termed the Meta-RL objective without adaptation as a multi-task objective). In Fig. \ref{fig:MAML_Multi} we find that Multi-Task Hyper-MAML outperforms the Hyper-MAML with adaptation. Moreover, Table \ref{run_time_table} shows that it also requires less training time as it removes the unnecessary complexity of the MAML adaptation training.

\begin{figure*}[hbt!]
\begin{center}
    \includegraphics[width=1.\linewidth]{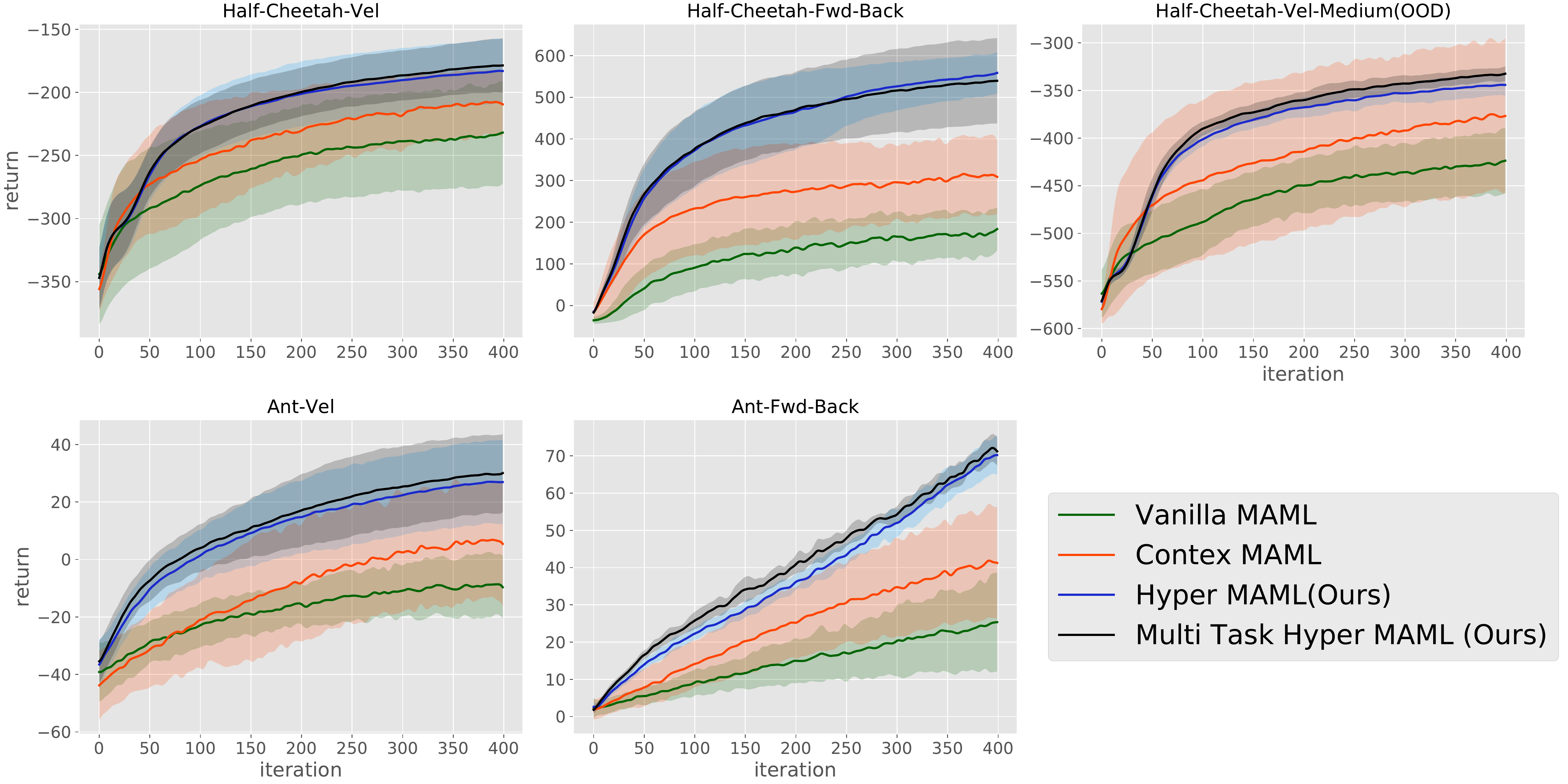}
\end{center}
\caption{Multi-Task Hyper-MAML performance over the \textbf{test tasks} with a Hypernetwork policy and a multi-task objective. Using a multi-task objective matches or outperforms the MAML objective without the need for the adaption step training.}
\label{fig:MAML_Multi}
\end{figure*}

\begin{table}[hbt!]
\caption{MAML highest rewards.}
\label{MAML_final_reward}
\begin{center}
\begin{tabular}{lccccc}
\multicolumn{1}{c}{\bf Algorithm}  &\multicolumn{1}{c}  {\bf Cheetah-Vel}  &\multicolumn{1}{c}{\bf Cheetah-Fwd-Back}    &\multicolumn{1}{c}{\bf Cheetah-Vel-Med}    &\multicolumn{1}{c}{\bf Ant-Vel}    &\multicolumn{1}{c}{\bf Ant-Fwd-Back}
\\ \hline \\
MAML        &$-231\pm40$       &$183\pm51$             &$-423\pm33$           &$-8\pm10$ &$25\pm13$ \\
Context MAML &$-207\pm25$       &$315\pm93$             &$-374\pm79$           &6$\pm19$ &$41\pm15$ \\
Hyper Multi-Task (Ours)   &$-\textbf{178}\pm\textbf{21}$       &$539\pm102$             &$-\textbf{332}\pm\textbf{7}$           &$\textbf{30}\pm\textbf{13}$ &$\textbf{72}\pm\textbf{3}$ \\
Hyper MAML (Ours) &$-182\pm25$       &$\textbf{558}\pm\textbf{49}$             &$-344\pm10$           &$27\pm14$ &$70\pm5$ \\ 

\end{tabular}
\end{center}
\end{table}

\begin{table}[hbt!]
\caption{MAML Hyperparameters}
\label{MAML_hyper_parameters}
\begin{center}
\begin{tabular}{lcc}
\multicolumn{1}{c}{\bf Hyperparameter}  &\multicolumn{1}{c}{\bf MAML}  &\multicolumn{1}{c}{\bf Hyper MAML (Ours)}
\\ \hline \\
Batch Size            &20 &20\\
Meta batch Size            &40 &40\\
Discount Factor            &0.95 &0.95 \\
Num of Iterations            &400 &400 \\
Max KL            &$1e^{-2}$ &$1e^{-2}$ \\
LS Max Steps    &20 &20 \\
Episode Max Steps    &200 &200 \\
\end{tabular}
\end{center}
\end{table}


\subsection{PEARL}

\begin{figure*}[hbt!]
\begin{center}
    \includegraphics[width=1.\linewidth]{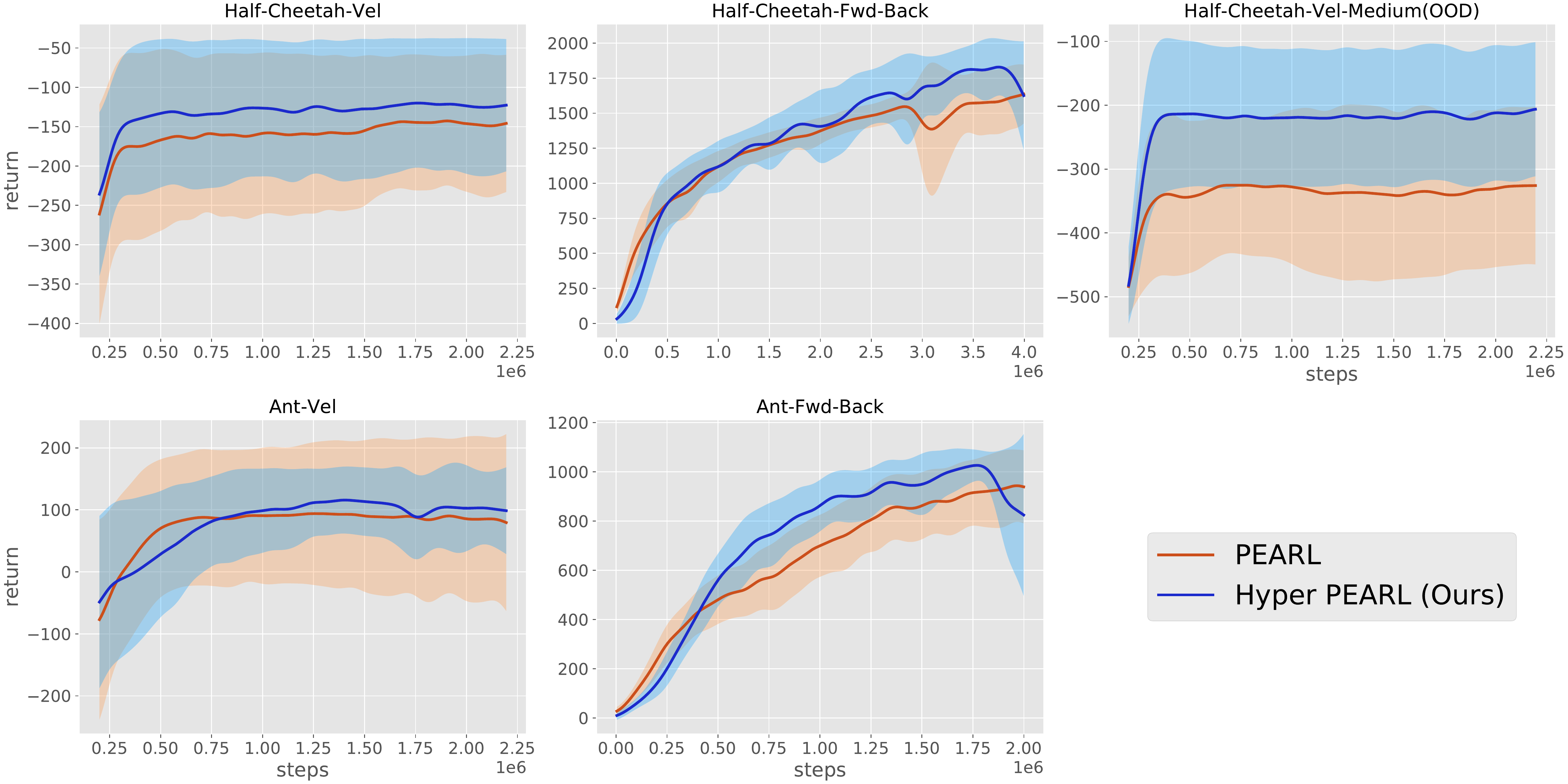}
\end{center}
\caption{PEARL Performance over the \textbf{test tasks} with policy and critic Hypernetworks compared to MLP policy and critic. Hypernetwork outperforms or matches MLP in all environments.}
\label{fig:PEARL}
\end{figure*}

\begin{figure*}[hbt!]
\begin{center}
    \includegraphics[width=1.\linewidth]{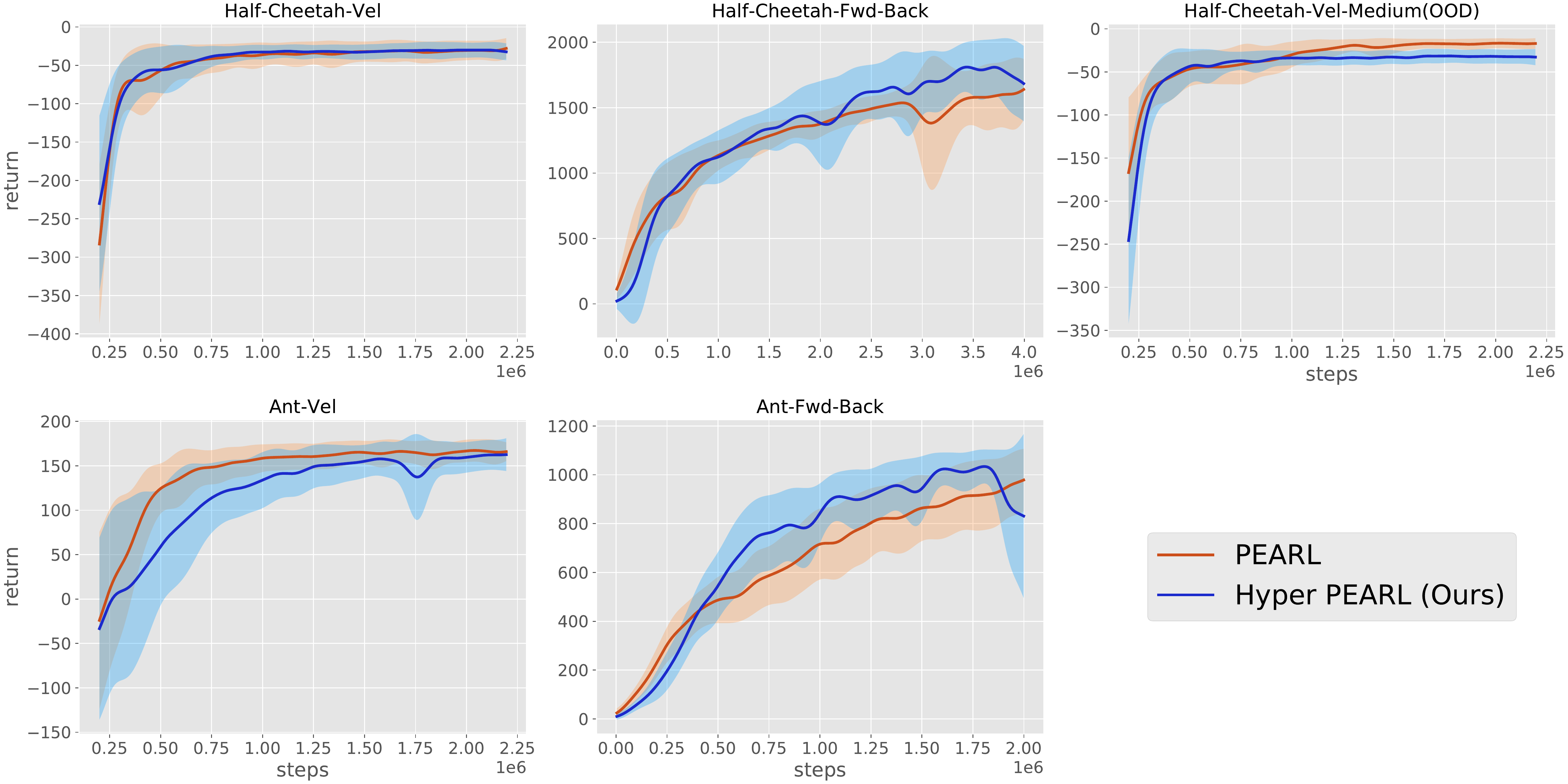}
\end{center}
\caption{PEARL Performance over \textbf{training tasks} with policy and critic Hypernetworks compared to MLP policy and critic.}
\label{fig:PEARL_train}
\end{figure*}

\begin{table}[hbt!]
\caption{PEARL highest rewards.}
\label{PEARL_final_reward}
\begin{center}
\begin{tabular}{lccccc}
\multicolumn{1}{c}{\bf Algorithm}  &\multicolumn{1}{c}  {\bf Cheetah-Vel}  &\multicolumn{1}{c}{\bf Ant-Vel}    &\multicolumn{1}{c}{\bf Cheetah-Vel-Med}    &\multicolumn{1}{c}{\bf Cheetah-Fwd-Back}    &\multicolumn{1}{c}{\bf Ant-Fwd-Back}
\\ \hline \\
PEARL  &$-142\pm82$       &$1636\pm210$             &$-325\pm109$           &$93\pm115$ &$943\pm146$ \\
Hyper PEARL (Ours) &$-\textbf{119}\pm\textbf{82}$       &$\textbf{1828}\pm\textbf{203}$             &$-\textbf{206}\pm\textbf{104}$           &$\textbf{115}\pm\textbf{54}$ &$\textbf{1026}\pm\textbf{62}$ \\

\end{tabular}
\end{center}
\end{table}

\begin{table}[hbt!]
\caption{PEARL Hyperparameters}
\label{PEARL_hyper_parameters}
\begin{center}
\begin{tabular}{lcc}
\multicolumn{1}{c}{\bf Hyperparameter} &\multicolumn{1}{c}{\bf PEARL}  &\multicolumn{1}{c}{\bf Hyper PEARL}
\\ \hline \\
Actor Learning Rate  &$3e^{-4}$ &$1e^{-4}$\\
Critic Learning Rate   &$3e^{-4}$  &$5e^{-5}$  \\
Context Learning Rate   &$3e^{-4}$  &$3e^{-4}$  \\
Value Learning Rate   &$3e^{-4}$  &$5e^{-5}$  \\
Optimizer             &Adam &Adam\\
Batch Size            &256 &256 \\
'Dir' Tasks Meta batch Size           &4 &4 \\
'Vel' Tasks Meta batch Size           &16 &16 \\
Target critic update     &0.005 &0.005\\
Discount Factor            &0.99 &0.99 \\
Num of Iterations            &400 &400 \\
Reward Scaling             &5 &5 \\
Episode Max Steps    &200 &200 \\
\end{tabular}
\end{center}
\end{table}




\end{document}